%% file: main.tex
\documentclass[runningheads]{llncs}

\usepackage{eccv}

\usepackage{eccvabbrv}

\usepackage{graphicx}
\usepackage{adjustbox}
\usepackage{booktabs}
\usepackage[accsupp]{axessibility}
\usepackage{microtype}
\usepackage{longtable}
\usepackage[section]{placeins}
\usepackage{algorithm}
\usepackage{algpseudocode}
\usepackage{tikz}
\usetikzlibrary{positioning,arrows.meta}

\newcommand{\citep}[1]{\cite{#1}}

\usepackage{hyperref}
\hypersetup{hypertexnames=false}

\newcommand{\cp}{CP\xspace}

\begin{document}

\title{SPARC: Single-Pass Scaling for Motion Forecasting with Conformal Bayesian Last Layers}
\titlerunning{Single-Pass Scaling for Motion Forecasting}

\author{%
Sakif Hossain\orcidID{0000-0003-0800-8392}\inst{1}\thanks{Equal contribution. Corresponding author.} \and
Julian Teusch\orcidID{0000-0002-4103-8430}\inst{1}\thanks{Equal contribution.} \and
J{\"o}rg P.\ M{\"u}ller\orcidID{0000-0001-7533-3852}\inst{1}%
}
\authorrunning{S.~Hossain et al.}

\institute{Clausthal University of Technology (TU Clausthal)\\
\email{sakif.hossain@tu-clausthal.de, julian.teusch@tu-clausthal.de}\\
\email{joerg.mueller@tu-clausthal.de}}

\maketitle

\input{paper_main.tex}

\bibliographystyle{splncs04}
\bibliography{conjbayes_hybrid_refs}

\end{document}

\title{Supplementary Material for: SPARC: Single-Pass Scaling for Motion Forecasting with Conformal Bayesian Last Layers}
\titlerunning{Supplementary Material}

\author{%
Sakif Hossain\orcidID{0000-0003-0800-8392}\inst{1}\thanks{Equal contribution. Corresponding author.} \and
Julian Teusch\orcidID{0000-0002-4103-8430}\inst{1}\thanks{Equal contribution.} \and
J{\"o}rg P.\ M{\"u}ller\orcidID{0000-0001-7533-3852}\inst{1}%
}
\authorrunning{S.~Hossain et al.}

\institute{Clausthal University of Technology (TU Clausthal)\\
\email{sakif.hossain@tu-clausthal.de, julian.teusch@tu-clausthal.de}\\
\email{joerg.mueller@tu-clausthal.de}}

\maketitle

\input{paper_supp.tex}

\bibliographystyle{splncs04}
\bibliography{conjbayes_hybrid_refs}

%% file: paper_main.tex
\begin{abstract}
			Human motion forecasters are increasingly accurate and fast, but reliable deployment requires uncertainty estimates that are structured, calibrated, and efficient.
			Bayesian and ensemble-based uncertainty estimates often require repeated stochastic inference~\citep{gal2016dropout,lakshminarayanan2017ensembles}, while conformal calibration alone does not provide an epistemic signal or preserve trajectory covariance structure~\citep{vovk2005conformal,fong2021conformalbayesian}.
			We introduce \emph{SPARC} (Single-Pass Adaptive Risk Calibration), a Bayesian--conformal uncertainty layer for motion forecasting.
			A deterministic MLP backbone predicts the future mean, and a conjugate Bayesian last layer converts time-domain feature leverage into an analytic horizon-wise epistemic scale $\kappa_t(x)$.
			This scale inflates a graph-temporal Gaussian covariance without changing its correlation structure, and split conformal calibration produces 95\% marginal prediction tubes with finite-sample validity under exchangeability.
			The key interface is the structured factorization $\kappa_t(x)\Sigma_{\mathrm{str},t}(x)$, which injects feature-space epistemic uncertainty into trajectory densities without Monte Carlo sampling.
			Across nine dataset/protocol blocks and deterministic, multimodal, and calibration baselines, \emph{SPARC} ranks first on NLL and on the combined MPJPE+NLL criterion while retaining competitive point accuracy and efficient calibrated tubes.
			Ranking windows by $\kappa$ separates high-error cases, making the scale usable as a lightweight risk monitor.
			\keywords{Human motion forecasting \and uncertainty quantification \and Bayesian last layer \and conformal prediction}
			\end{abstract}

\section{Introduction}
Human motion forecasting---predicting future 3D pose trajectories from a short observed prefix---is a core primitive for human--robot interaction, animation, and anticipatory perception.
Recent forecasters achieve low mean per-joint position error (MPJPE) with fast inference~\citep{martinez2017rnn,guo2023simlpe}, but point forecasts alone are insufficient for downstream decision-making.
Reliable systems must reason about \emph{uncertainty}: multiple futures may be plausible due to intrinsic ambiguity (aleatoric), and models may become unreliable when extrapolating beyond their training support (epistemic), especially under distribution shift~\citep{kendall2017uncertainty,ovadia2019trust}.

Bayesian deep learning models epistemic uncertainty through a posterior over model parameters, yet practical approaches often rely on approximate inference and repeated stochastic forward passes (e.g., Monte Carlo (MC) dropout~\citep{gal2016dropout} or ensembles~\citep{lakshminarayanan2017ensembles}) and can degrade under shift~\citep{ovadia2019trust}.
Moreover, modern neural networks are frequently miscalibrated even in independent and identically distributed (i.i.d.) settings, motivating calibration layers with explicit guarantees~\citep{guo2017calibration}.
		Neural-linear and Bayesian last-layer methods reduce this cost by keeping a deterministic feature extractor and placing a conjugate Bayesian model on the final linear layer, yielding closed-form posterior predictive uncertainty at near single-pass cost~\citep{mackay1992bayesian,riquelme2018deepbayesianbandits,thakur2020luna,watson2021ldbayesll,daxberger2021laplaceredux}.
			However, neural-linear uncertainty is typically paired with factorized (diagonal or isotropic) observation models~\citep{mackay1992bayesian,daxberger2021laplaceredux,kim2025claps}, making it nontrivial to preserve structured spatio-temporal correlations for trajectory outputs.

			In parallel, conformal prediction provides distribution-free predictive inference with finite-sample marginal validity under exchangeability~\citep{vovk2005conformal,lei2018distributionfree,angelopoulos2021gentle}.
			Conformal methods have also been extended to functional and trajectory-like outputs, yielding finite-sample valid prediction bands around a base forecaster~\citep{lei2013conformalfunctional,diquigiovanni2021conformalbandsmfd}.
			For regression, conformalization can wrap any base predictor (including heteroscedastic quantile regressors~\citep{romano2019cqr}) to guarantee target coverage, but it does not by itself separate epistemic from aleatoric uncertainty, and achieving strong conditional guarantees remains challenging in general~\citep{gibbs2023conditionalconformal}.

				We introduce \emph{SPARC}, a Bayesian--conformal interface that adds structured calibrated uncertainty to a fast deterministic motion forecaster in a single pass.
				For deployment, uncertainty must attach to future joint trajectories with temporal and skeletal dependencies, not only to a scalar confidence score.
				Starting from a strong deterministic MLP backbone, SPARC places a conjugate Bayesian model on the final linear layer and derives an analytic, input-dependent \emph{epistemic scale} $\kappa(x)$ (or timewise $\kappa_t(x)$) with a leverage-style interpretation.
				This scale multiplicatively inflates a structured Gaussian aleatoric covariance that models spatio-temporal and skeletal correlations, and split conformal calibration converts the resulting distribution into 95\% marginal prediction tubes under exchangeability~\citep{lei2013conformalfunctional,diquigiovanni2021conformalbandsmfd}, complementing Bayesian credible regions under model misspecification~\citep{fong2021conformalbayesian,zhang2024posteriorcp}.
				The same $\kappa$ signal also enables $\kappa$-conditioned (Mondrian) tube scaling and optional trajectory-level conformal ellipsoid scores (Sec.~\ref{sec:split_conformal}).

\paragraph{Contributions.}
				We contribute:
				\begin{itemize}
				  \setlength{\itemsep}{0.15em}
				  \setlength{\topsep}{0.15em}
				  \setlength{\parskip}{0pt}
				  \item \textbf{Structured Bayesian last-layer scaling.} Under a multi-output conjugate model, the predictive covariance factorizes exactly as $\kappa_t(x)\,\Sigma_{\mathrm{str},t}(x)$ (Prop.~\ref{prop:structured_scaling}), enabling analytic epistemic inflation while preserving arbitrary positive-definite trajectory covariance structure.
				  \item \textbf{Single-pass Bayesian--conformal motion uncertainty.} We combine this scale with a graph-temporal Gaussian covariance head and split conformal calibration, producing calibrated marginal prediction tubes without Monte Carlo sampling or ensembles.
				  \item \textbf{Cross-dataset probabilistic gains.} Across nine dataset/protocol blocks, \emph{SPARC} ranks first on average NLL and on overall MPJPE+NLL, with efficient calibrated tubes and competitive point accuracy.
				  \item \textbf{Deployment-oriented epistemic monitoring.} On Human3.6M, ranking windows by $\kappa$ separates high-error cases: the highest-$\kappa$ decile has $1.79{\times}$ MPJPE compared to the average, and filtering that decile reduces MPJPE by $\approx 9\%$ (supplementary Sec.~\emph{Epistemic scale diagnostics}; thresholds tuned on held-out data, with no formal guarantee implied).
					\end{itemize}

\section{Related Work}
Human motion forecasting has long been dominated by deterministic point predictors, ranging from early recurrent/seq2seq models~\citep{martinez2017rnn} to modern feed-forward backbones that offer strong runtime--accuracy trade-offs~\citep{guo2023simlpe}.
Recent deterministic improvements add structure via graph convolutions~\citep{dang2022msr}, auxiliary supervision~\citep{xu2023auxiliary_tasks}, stability-inducing integration~\citep{chen2024symplectic}, or deviation-feedback refinement~\citep{sun2023defeenet}.
To represent the inherent ambiguity of future motion, multimodal predictors pursue latent-variable and generative formulations, including normalizing flows~\citep{yuan2020dlow}, diverse candidate generation~\citep{mao2021gsps}, and diffusion-style denoising models~\citep{chen2023humanmac}.
While these approaches can capture multiple plausible futures, their inference often involves sampling or multiple hypotheses, which motivates lightweight uncertainty layers that preserve real-time point-forecasting backbones~\citep{yuan2020dlow,mao2021gsps,chen2023humanmac}.

Uncertainty quantification for deep regression typically distinguishes aleatoric ambiguity from epistemic uncertainty due to limited training support~\citep{kendall2017uncertainty}.
Common epistemic estimators such as MC dropout and deep ensembles~\citep{gal2016dropout,lakshminarayanan2017ensembles} can improve robustness but require repeated stochastic forward passes and can still fail under distribution shift~\citep{ovadia2019trust}.
Neural-linear and last-layer Bayesian approximations keep the feature extractor deterministic and place Bayesian inference on the final linear layer, yielding analytic predictive uncertainty at near single-pass cost~\citep{mackay1992bayesian,riquelme2018deepbayesianbandits,thakur2020luna,watson2021ldbayesll,daxberger2021laplaceredux}.
However, these epistemic estimators are frequently used with factorized (diagonal or isotropic) likelihoods in regression~\citep{mackay1992bayesian,daxberger2021laplaceredux,kim2025claps}, which does not capture spatio-temporal or skeletal correlations needed for coherent trajectory tubes and likelihood.

				Structured Gaussian models such as matrix-normal/Kronecker factorizations capture separable spatio-temporal correlations, while graph/Gaussian Markov random field (GMRF) constructions encode skeletal coupling via sparse precisions~\citep{dawid1981matrixvariate,besag1974spatial,rue2005gmrf}.
				Separately, conformal prediction offers distribution-free predictive inference with finite-sample \emph{marginal} validity under exchangeability~\citep{vovk2005conformal,lei2018distributionfree,angelopoulos2021gentle}.
				It has been extended to functional and trajectory bands~\citep{lei2013conformalfunctional,diquigiovanni2021conformalbandsmfd}; for regression, conformalized quantile regression is a strong calibration baseline~\citep{romano2019cqr}, while dependent time series require additional assumptions or online adaptation~\citep{xu2023spci}.
				Recent work adapts conformal calibration to structured prediction, scene graphs, and multi-dimensional time series, illustrating that for non-scalar outputs the task-specific score and output interface are central~\citep{zhang2025conformalstructured,nag2025scenegraphcp,lee2026flowbased}.
				Conformal methods have also been used for trajectory uncertainty under distribution shift and safety-aware planning~\citep{huang2024cuqds,lindemann2022safeplanning}; in 3D human motion forecasting, latent conformal prediction calibrates multidimensional trajectories~\citep{ma2026latentconformal}.
				Approximate Bayesian posteriors have likewise been paired with conformal layers, including aleatoric--epistemic scaling via last-layer Laplace for conformal regression~\citep{kim2025claps} and posterior conformal prediction~\citep{zhang2024posteriorcp}.
				SPARC jointly integrates analytic conjugate Bayesian $\kappa_t$, graph-temporal covariance, marginal conformal tubes, and optional trajectory-level conformal set scores in a single-pass motion forecasting pipeline.

\section{Method}
\label{sec:method}
SPARC turns a deterministic forecaster into calibrated marginal prediction tubes through a Bayesian last-layer scale and a post-hoc conformal layer.
Figure~\ref{fig:arch_hyb} shows the dependency structure.
We then specify the calibrated deployment recipe used in experiments.
\begin{figure}[!t]
  \centering
  \IfFileExists{tikz/arch_conjbayes_hybrid.tex}{%
    \resizebox{0.95\linewidth}{!}{\input{tikz/arch_conjbayes_hybrid.tex}}%
  }{%
    \fbox{\parbox{0.95\linewidth}{Missing figure: tikz/arch_conjbayes_hybrid.tex}}%
  }
  \caption{SPARC dependency view. (A) Core: features yield $\kappa_t=1+\phi_t^\top\Lambda_{n,t}^{-1}\phi_t$ and inflate the structured covariance. (B) Split CP: calibration scores produce per-joint quantiles $q_{t,j}$ for marginal tubes.}
  \label{fig:arch_hyb}
\end{figure}
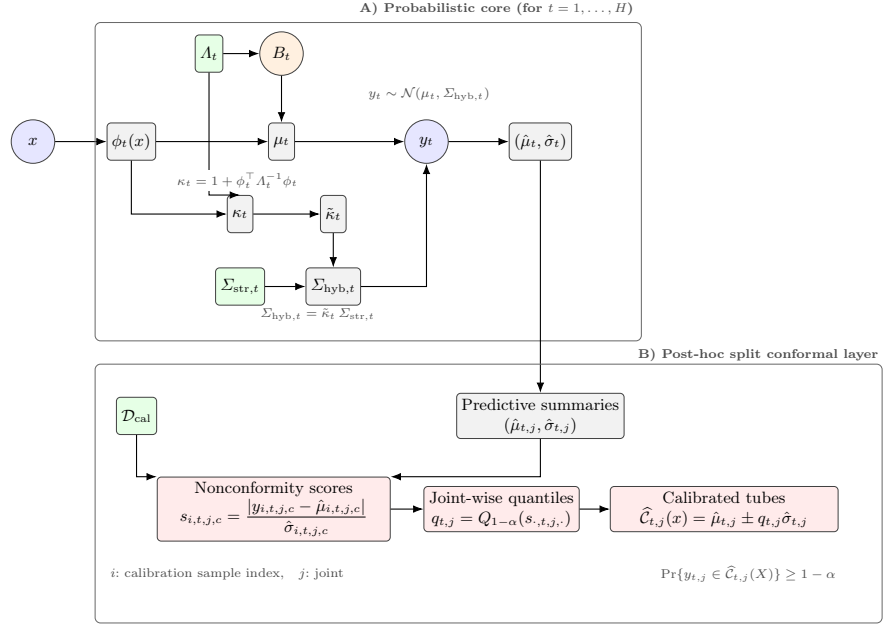
Algorithm~\ref{alg:conjgraph_cp} summarizes the recipe for adding this uncertainty layer to deterministic forecasters.
We additionally describe two conformal variants enabled by the analytic $\kappa$ and structured covariance: $\kappa$-conditioned (Mondrian) tubes and a trajectory-level conformal ellipsoid (Sec.~\ref{sec:split_conformal}).
Simpler variants in our ablation are recovered by dropping the graph coupling, the timewise shrinkage, or the hybrid inflation.

\paragraph{Method overview.}
SPARC combines a deterministic mean backbone $\hat\mu$ with a structured Gaussian head for aleatoric temporal/skeletal correlations $\Sigma_{\mathrm{str}}$.
A conjugate Bayesian last layer maps time-domain feature leverage to $\kappa_t(x)$, and split conformal calibration turns standardized residuals into marginal tubes; inference uses one forward pass plus closed-form quadratic forms.

\subsection{Hybrid structured Gaussian predictor (ConjGraph)}
\label{sec:conjgraph}
Let $x_{1:T}\in\mathbb{R}^{T\times C}$ be an observed pose prefix and $y_{1:H}\in\mathbb{R}^{H\times C}$ the future trajectory, with $C=3J$ coordinates (flattened joints), where $J$ is the number of joints.
As in our implementation, we train and evaluate on \emph{displacements} relative to the last observed pose $x_T$ (i.e., $y_t\leftarrow y_t-x_T$); absolute poses are recovered by adding $x_T$ back at inference.

\paragraph{Goal: calibrated marginal tubes.}
Our primary reported object is a collection of axis-aligned marginal intervals, written compactly as a prediction tube $\mathcal{C}_\alpha(x)$ with miscoverage $\alpha\in(0,1)$:
\begin{equation}
  \widehat{\mathcal{C}}_\alpha(x)
  \;=\;
  \prod_{t=1}^H\prod_{j=1}^J\prod_{d=1}^3
  \bigl[\hat\mu_{t,j,d}(x)\ \pm\ q_{t,j}\,\hat\sigma^{\mathrm{CP}}_{t,j,d}(x)\bigr],
  \label{eq:tube_def}
\end{equation}
where $(\hat\mu,\hat\sigma^{\mathrm{CP}})$ come from a probabilistic predictor and $\{q_{t,j}\}$ are split-conformal calibration factors (Sec.~\ref{sec:split_conformal}).
The product notation denotes the reported rectangle of marginal intervals; the guarantee below is marginal for the induced score distribution, while stricter trajectory-level set scores are considered separately (Sec.~\ref{sec:split_conformal}).

\paragraph{Backbone and time-domain design vector.}
We follow the siMLPe/DCT-bridge backbone~\citep{guo2023simlpe}: the input sequence is transformed along the time axis into a discrete cosine transform (DCT) coefficient representation and passed through a deterministic trunk $f_\theta$ that outputs coefficient-wise features $h_n(x)\in\mathbb{R}^{C}$ for $n\in\{1,\dots,N\}$.
Let $A\in\mathbb{R}^{N\times N}$ denote the inverse DCT (iDCT) matrix that maps coefficient outputs back to the time domain.
In the fully-connected output (FC-Out) parameterization used throughout this work, the time-domain mean at horizon step $t$ is linear in an augmented design vector $\phi_t(x)\in\mathbb{R}^{P}$:
\begin{align}
  \hat\mu_t(x) &= B_t^\top \phi_t(x), \label{eq:mu_time}\\
  \phi_t(x) &= \begin{bmatrix} g_t(x) \\ s_t \end{bmatrix}\in\mathbb{R}^{P},\qquad P=C+1,\nonumber\\
  g_t(x) &= \sum_{n=1}^N A_{t,n}\,h_n(x)\in\mathbb{R}^{C},\qquad
  s_t=\sum_{n=1}^N A_{t,n}\in\mathbb{R}. \label{eq:phi_def}
\end{align}
The scalar $s_t$ accounts for the coefficient-space bias after mapping back through the iDCT.

\paragraph{Conjugate Bayesian last-layer scale.}
We place a conjugate Bayesian model on the final linear map in~\eqref{eq:mu_time} while keeping the trunk deterministic.
For each horizon step $t$, consider the multi-output linear-Gaussian model
\begin{equation}
  y_t \;=\; B_t^\top \phi_t(x) + \varepsilon_t,
  \qquad
  \varepsilon_t \sim \mathcal{N}(0,\Sigma_{\mathrm{str},t}(x)),
  \label{eq:lin_gauss}
\end{equation}
	with a matrix-normal prior~\citep{dawid1981matrixvariate}
	\begin{equation}
	  B_t \mid \Sigma_{\mathrm{str},t}
	  \;\sim\;
	  \mathcal{MN}(B_{0,t},\Lambda_0^{-1},\Sigma_{\mathrm{str},t}),
	  \qquad
	  \Lambda_0=\lambda_0 I_P.
	  \label{eq:mn_prior}
	\end{equation}
	
		\begin{proposition}[Exact structured-covariance scaling]
		\label{prop:structured_scaling}
		Conditioning on any positive-definite $\Sigma_{\mathrm{str},t}(x)$, integrating out $B_t$ under the conjugate matrix-normal model in~\eqref{eq:lin_gauss}--\eqref{eq:mn_prior} preserves the structured covariance up to a \emph{scalar} multiplicative inflation:
	\[
	  y_t \mid x,\mathcal{D} \;\sim\; \mathcal{N}\!\bigl(\mu^{\mathrm{post}}_t(x),\ \kappa_t(x)\,\Sigma_{\mathrm{str},t}(x)\bigr),
	  \qquad
	  \kappa_t(x) = 1 + \phi_t(x)^\top \Lambda_{n,t}^{-1} \phi_t(x),
	\]
	where $\Lambda_{n,t}=\Lambda_0 + \sum_{i\in\mathcal{D}}\phi_t(x_i)\phi_t(x_i)^\top$.
	In particular, $\kappa_t$ rescales \emph{only} the covariance magnitude and preserves the correlation/structure encoded by $\Sigma_{\mathrm{str},t}$.
	\end{proposition}
	The intuition is that integrating over the last-layer weights adds uncertainty in the output direction induced by the feature leverage $\phi_t(x)^\top\Lambda_{n,t}^{-1}\phi_t(x)$.
	Because the matrix-normal prior shares the same output covariance as the structured residual model, this additional uncertainty has the same correlation structure as $\Sigma_{\mathrm{str},t}$.
	Thus epistemic uncertainty can be injected by the scalar multiplier $\kappa_t(x)$ instead of replacing the structured covariance with a diagonal approximation.
	
	\begin{proof}
	By conjugacy, $B_t \mid \Sigma_{\mathrm{str},t},\mathcal{D}\sim\mathcal{MN}(B_{n,t},\Lambda_{n,t}^{-1},\Sigma_{\mathrm{str},t})$, hence $\mathrm{vec}(B_t)\mid \Sigma_{\mathrm{str},t},\mathcal{D}$ is Gaussian with covariance $\Sigma_{\mathrm{str},t}\otimes \Lambda_{n,t}^{-1}$.
	Therefore $B_t^\top \phi_t(x)$ has covariance $(\phi_t(x)^\top \Lambda_{n,t}^{-1}\phi_t(x))\,\Sigma_{\mathrm{str},t}(x)$, and adding $\varepsilon_t\sim\mathcal N(0,\Sigma_{\mathrm{str},t}(x))$ yields the stated inflation.
	\end{proof}
	
		\begin{remark}[Arbitrary structured covariance]
		The factorization in Prop.~\ref{prop:structured_scaling} holds for any positive-definite $\Sigma_{\mathrm{str},t}(x)$ (dense, Kronecker-factored, or defined via a sparse precision), enabling epistemic inflation without simplifying the structured aleatoric head.
			\end{remark}
			We use $\kappa_t(x)$ as an analytic epistemic inflation factor (no MC sampling) that increases when the test feature vector $\phi_t(x)$ has large leverage under the training design.
			This is a feature-space signal: its informativeness depends on the learned representation $\phi_t$ and should not be interpreted as a universal out-of-distribution detector independent of the backbone.
			In SPARC we keep the backbone mean predictor $\hat\mu_t(x)$ and use only $\kappa_t(x)$ (and its stabilized version $\tilde\kappa_t(x)$ below) to rescale the structured covariance head.

	\paragraph{Closed-form fitting and stabilization.}
	We only require $\Lambda_{n,t}^{-1}$ to compute $\kappa_t$.
	With $\Lambda_0=\lambda_0 I_P$, we fit $\Lambda_{n,t}^{-1}$ in closed form by accumulating the sufficient statistic $S_{xx,t}=\sum_i \phi_t(x_i)\phi_t(x_i)^\top$ on the training set and computing $(S_{xx,t}+\lambda_0 I_P)^{-1}$ once via a Cholesky factorization.
	At inference, we evaluate the quadratic form $\kappa_t(x)=1+\phi_t(x)^\top\Lambda_{n,t}^{-1}\phi_t(x)$ directly, avoiding MC sampling.
	To stabilize horizon-wise scaling, we optionally shrink $\kappa_t$ toward a time-homogeneous sequence-level scale derived from pooled statistics.
	Let $\Lambda_{\mathrm{glob}}=\lambda_0 I_P+\sum_{t,i}\phi_t(x_i)\phi_t(x_i)^\top$ and define
\begin{equation}
  \bar\kappa(x)\;=\;\frac{1}{H}\sum_{t=1}^H \Bigl(1+\phi_t(x)^\top \Lambda_{\mathrm{glob}}^{-1}\phi_t(x)\Bigr).
  \label{eq:kappa_bar}
\end{equation}
Our stabilized scale is
\begin{equation}
  \tilde\kappa_t(x)\;=\;(1-\rho)\,\kappa_t(x)+\rho\,\bar\kappa(x),\qquad \rho\in[0,1].
  \label{eq:kappa_shrink}
\end{equation}
The global-$\kappa$ variant is recovered by setting $\rho=1$.
In experiments we also allow a monotone ``epistemic temperature'' transform $\tilde\kappa_t\leftarrow \tilde\kappa_t^{\gamma}$ with $\gamma>0$ (with $\gamma=1$ corresponding to the default construction) to tune the NLL/interval-width (mean marginal interval width; reported as W95) trade-off.

\paragraph{Per-joint leverage scale for \cp tubes.}
The global timewise scale $\tilde\kappa_t(x)$ is shared across all joints at horizon $t$.
To allow per-joint epistemic inflation in the \cp layer without changing the likelihood, we compute an additional per-joint scale $\kappa_j(x)$ from reduced per-joint design vectors.
Let $\phi_{t,j}(x)\in\mathbb{R}^{4}$ contain the 3 coordinates of joint $j$ in $g_t(x)$ plus the shared bias feature $s_t$.
We fit a per-joint conjugate statistic $\Lambda^{-1}_{n,j}=(\lambda_{0,\mathrm{joint}}I+\sum_{i,t}\phi_{t,j}(x_i)\phi_{t,j}(x_i)^\top)^{-1}$ on $\mathcal D_{\mathrm{train}}$ and define $\kappa_{t,j}(x)=1+\phi_{t,j}(x)^\top\Lambda^{-1}_{n,j}\phi_{t,j}(x)$, aggregated as $\kappa_j(x)=\frac1H\sum_{t=1}^H \kappa_{t,j}(x)$.
For conformal tubes we use
\begin{equation}
  \hat\sigma^{\mathrm{CP}}_{t,j,d}(x)
  \;=\;
  \sqrt{\kappa_j(x)}\,\hat\sigma_{t,j,d}(x),
  \label{eq:sigma_cp_jointkappa}
\end{equation}
while the probabilistic core (and NLL) remains unchanged.

\paragraph{Structured Gaussian aleatoric covariance (MN-GraphJ).}
We refer to this matrix-normal + joint-graph construction as \emph{MN-GraphJ} (MatrixNormal-GraphJ).
The backbone predicts a structured \emph{aleatoric} covariance for the residual matrix
$R(x)=y_{1:H}-\hat\mu_{1:H}(x)\in\mathbb{R}^{H\times C}$.
We use a separable matrix-normal parameterization,
\begin{equation}
  \begin{aligned}
    \mathrm{vec}(R(x)) &\sim \mathcal{N}\!\bigl(0,\ \Sigma_T(x)\otimes\Sigma_C(x)\bigr), \\
    R(x) &\sim \mathcal{MN}(0,\Sigma_T(x),\Sigma_C(x)),
  \end{aligned}
  \label{eq:mn_ale}
\end{equation}
where $\Sigma_T(x)\in\mathbb{R}^{H\times H}$ captures temporal correlations and $\Sigma_C(x)\in\mathbb{R}^{C\times C}$ captures spatial/skeletal coupling.
We parameterize $\Sigma_T(x)=L_T(x)L_T(x)^\top$ with a learned lower-triangular $L_T(x)$ (packed Cholesky representation), enabling dense temporal correlations with stable optimization.
To encode skeletal structure, we model $\Sigma_C$ via a joint-graph precision (a GMRF construction~\citep{besag1974spatial,rue2005gmrf}).
Let $L_{\mathrm{joint}}\in\mathbb{R}^{J\times J}$ be the unweighted graph Laplacian of a chosen joint graph (e.g., the kinematic tree, or a data-driven $k$-nearest-neighbor (kNN) graph estimated from training motions).
The network predicts positive scalars $\tau(x)$ and $\epsilon(x)$, and we define
\begin{equation}
  Q_J(x) \;=\; \tau(x)\,L_{\mathrm{joint}} + \epsilon(x)\,I_J,
  \qquad
  \Sigma_C(x) \;=\; (Q_J(x)\otimes I_3)^{-1}.
  \label{eq:graph_cov}
\end{equation}

\paragraph{Hybridization.}
We combine epistemic and aleatoric components by inflating the temporal covariance factor with $\tilde\kappa_t(x)$.
Define
\begin{equation}
  D_\kappa(x)=\operatorname{diag}\!\bigl(\sqrt{\tilde\kappa_t(x)}\bigr)_{t=1}^H,
  \label{eq:d_kappa}
\end{equation}
and set
\begin{equation}
  \Sigma_T^{\mathrm{hyb}}(x)
  \;=\;
  D_\kappa(x)\,\Sigma_T(x)\,D_\kappa(x)^\top,
  \qquad
  \Sigma^{\mathrm{hyb}}(x)=\Sigma_T^{\mathrm{hyb}}(x)\otimes \Sigma_C(x).
  \label{eq:hyb_cov}
\end{equation}
If $\Sigma_T(x)=L_T(x)L_T(x)^\top$, we implement~\eqref{eq:hyb_cov} by scaling the $t$-th row of $L_T(x)$ by $\sqrt{\tilde\kappa_t(x)}$, preserving correlation structure while inflating scale.
The resulting marginal variance factorizes as
\begin{equation}
  \hat\sigma^2_{t,c}(x)
  \;=\;
  \mathrm{Var}(y_{t,c}\mid x)
  \;=\;
  \tilde\kappa_t(x)\,[\Sigma_T(x)]_{tt}\,[\Sigma_C(x)]_{cc}.
	  \label{eq:marg_var}
\end{equation}

\paragraph{Scope of the exact result and stabilizers.}
Prop.~\ref{prop:structured_scaling} establishes the exact last-layer covariance factorization under the stated conjugate Gaussian model, conditional on the structured covariance.
The shrinkage coefficient $\rho$ and epistemic temperature $\gamma$ are held-out stabilizers for the NLL--width trade-off, not additional assumptions in the theorem.
Likewise, the per-joint scale in~\eqref{eq:sigma_cp_jointkappa} is used only for conformal tube scoring and construction; it does not change the Gaussian NLL core.
The split-conformal guarantee below therefore attaches to the final calibrated score under exchangeability, while the Bayesian factorization explains the probabilistic core that supplies the scale.

\subsection{Split conformal calibration for prediction tubes}
\label{sec:split_conformal}
Bayesian credible regions can be miscalibrated under model misspecification or dataset shift~\citep{fong2021conformalbayesian,zhang2024posteriorcp}.
We therefore conformalize the hybrid Gaussian predictor to obtain calibrated marginal tubes with finite-sample validity under exchangeability~\citep{vovk2005conformal,lei2018distributionfree,angelopoulos2021gentle}.

Following split conformal calibration~\citep{vovk2005conformal,lei2018distributionfree,angelopoulos2021gentle}, on a held-out calibration set $\mathcal{D}_{\mathrm{cal}}=\{(x_i,y_i)\}_{i=1}^{n}$, we compute standardized absolute residual scores
\begin{equation}
  s_{i,t,j,d} \;=\; \frac{|y_{i,t,j,d}-\hat\mu_{t,j,d}(x_i)|}{\hat\sigma^{\mathrm{CP}}_{t,j,d}(x_i)},
\end{equation}
where $\hat\sigma^{\mathrm{CP}}$ is the scale used for tube construction (in SPARC this includes the hybrid timewise inflation $\tilde\kappa_t$ and the \cp-only joint factor in~\eqref{eq:sigma_cp_jointkappa}).
For each horizon step $t$ and joint $j$, we set $q_{t,j}$ using the split-conformal ``higher'' quantile rule,
\begin{equation}
  q_{t,j} \;=\;
  \mathrm{Quantile}\!\left(\{s_{i,t,j,d}\}_{i,d},\ \left\lceil (N_{t,j}+1)(1-\alpha)\right\rceil/N_{t,j}\right),
\end{equation}
where $N_{t,j}$ is the number of pooled scores for $(t,j)$ (in our implementation $N_{t,j}=n\cdot 3$).
We then output the tube in~\eqref{eq:tube_def}.

\paragraph{Guarantee and practical notes.}
Split conformal provides finite-sample \emph{marginal} validity for the induced score distribution under exchangeability between calibration and evaluation windows.
Calibration is post-hoc: it computes quantiles on $\mathcal D_{\mathrm{cal}}$ and does not update network weights, covariance parameters, or the conjugate statistics.
Under temporal dependence or distribution shift that violates exchangeability, the conformal layer should be read as empirical calibration rather than an unconditional coverage guarantee.
Pooling over the 3 coordinates per joint targets the mean marginal coverage metric we report; stricter max-score variants are more conservative.
\paragraph{Risk-adaptive and set-valued variants.}
The supplement reports two extensions: $\kappa$-conditioned (Mondrian) conformal tubes that calibrate bin-wise $q_{t,j,b}$~\citep{vovk2012conditional,bostrom2020mondrian}, and a trajectory-level conformal ellipsoid based on a structured Mahalanobis/whitened-residual score under $\Sigma_T^{\mathrm{hyb}}\otimes\Sigma_C$.

\begin{algorithm}[t]
\caption{SPARC: single-pass structured uncertainty with analytic $\kappa$ and split conformal calibration}
\label{alg:conjgraph_cp}
\begin{algorithmic}[1]
\Require Training data $\mathcal{D}_{\mathrm{train}}$, calibration data $\mathcal{D}_{\mathrm{cal}}$, miscoverage $\alpha$, prior precisions $\lambda_0$ and $\lambda_{0,\mathrm{joint}}$
\State Train a deterministic mean forecaster (mean fine-tune; MeanFT) to obtain time-domain features $\phi_t(x)$ and mean predictions $\hat\mu_t(x)$.
\State Fit a structured Gaussian covariance head (e.g., MN-GraphJ) on residuals with frozen mean, yielding $\Sigma_{\mathrm{str},t}(x)$.
\State Accumulate $S_{xx,t}\leftarrow \sum_{(x_i,\cdot)\in\mathcal{D}_{\mathrm{train}}}\phi_t(x_i)\phi_t(x_i)^\top$ and compute $\Lambda_{n,t}^{-1}=(\lambda_0 I + S_{xx,t})^{-1}$ (e.g., via Cholesky factorization).
\State Fit per-joint leverage statistics: accumulate $S_{xx,j}\leftarrow \sum_{(x_i,\cdot)\in\mathcal{D}_{\mathrm{train}},t}\phi_{t,j}(x_i)\phi_{t,j}(x_i)^\top$ and compute $\Lambda_{n,j}^{-1}=(\lambda_{0,\mathrm{joint}} I + S_{xx,j})^{-1}$. Use it to compute $\kappa_j(x)$ and $\hat\sigma^{\mathrm{CP}}$ via~\eqref{eq:sigma_cp_jointkappa}.
\State \textbf{Split conformal calibration:} on $\mathcal{D}_{\mathrm{cal}}$, compute
\Statex \hspace{\algorithmicindent}$s_{i,t,j,d} = |y_{i,t,j,d}-\hat\mu_{t,j,d}(x_i)|/\hat\sigma^{\mathrm{CP}}_{t,j,d}(x_i)$ and set
\Statex \hspace{\algorithmicindent}$q_{t,j} \leftarrow \mathrm{Quantile}(\{s_{i,t,j,d}\}_{i,d}, \lceil (N_{t,j}+1)(1-\alpha)\rceil/N_{t,j})$ (higher), where $N_{t,j}$ is the number of pooled scores for $(t,j)$.
\State \textbf{(Optional) $\kappa$-Mondrian \cp:} bin sequences by $r(x)$ (e.g., mean $\sqrt{\kappa_t(x)}$) and compute $q_{t,j,b}$ per bin $b$; at test time use $q_{t,j,b(x)}$.
\State \textbf{(Optional) trajectory \cp:} compute $s_i=\|{L_T^{\mathrm{hyb}}(x_i)}^{-1}(y_i-\hat\mu(x_i))L_C(x_i)^{-\top}\|_F$ and conformalize to get $q_{\mathrm{traj}}$, yielding an ellipsoidal set $\{y:\ s(y)\le q_{\mathrm{traj}}\}$.
\State \textbf{Prediction for new $x$:} compute $\kappa_t(x)=1+\phi_t(x)^\top\Lambda_{n,t}^{-1}\phi_t(x)$, optionally shrink to $\tilde\kappa_t(x)$ via~\eqref{eq:kappa_shrink}, and form $\Sigma_{\mathrm{hyb},t}(x)=\tilde\kappa_t(x)\Sigma_{\mathrm{str},t}(x)$.
\State Output calibrated tube $\widehat{\mathcal{C}}_\alpha(x)$ in~\eqref{eq:tube_def} (or $q_{t,j,b(x)}$), and optionally the trajectory ellipsoid.
\end{algorithmic}
\end{algorithm}

\section{Experiments}
\subsection{Experimental setup}
\paragraph{Datasets.}
We treat Human3.6M (H36M)~\citep{ionescu2014human36m} as the primary benchmark (indoor motion capture (MoCap)) and additionally evaluate cross-dataset generalization on AMASS~\citep{mahmood2019amass}, LaFAN1~\citep{harvey2020robust}, CMU-MoCap~\citep{cmu_mocap}, 3DPW~\citep{vonmarcard2018pw3d}, and three human--robot interaction (HRI)-style datasets with different capture / interaction setups: CHICO~\citep{sampieri2022chico}, HA4M~\citep{cicirelli2022ha4m}, and AnDy~\citep{maurice2018andy}.
Unless explicitly stated otherwise, non-H36M datasets use $T{=}50$ observed frames and we forecast $H{=}25$ future frames.

\paragraph{Evaluation protocol.}
We use a random-but-deterministic partition of the official evaluation windows.
We fix the random number generator (RNG) seed to 304 and permute the windows.
We use the first $n_{\mathrm{cal}}{=}512$ windows for split conformal calibration and report all metrics on the next $n_{\mathrm{eval}}{=}1024$ windows.
We use $\alpha{=}0.05$ throughout.
We keep $(n_{\mathrm{cal}},n_{\mathrm{eval}},\alpha)$ fixed across models.
Data usage is separated by operation: the training split fits the backbone, covariance heads, and conjugate sufficient statistics; held-out training windows select the five uncertainty/calibration hyperparameters: $\lambda_0$ (`lambda0'), $\rho$ (`time\_shrink\_rho'), $\gamma$ (`kappa\_power'), $\lambda_{0,\mathrm{joint}}$ (`lambda0\_joint'), and $\gamma_{\mathrm{joint}}$ (`joint\_kappa\_power').
The calibration subset computes only split-conformal quantiles $\{q_{t,j}\}$, and the disjoint evaluation subset is used only once for the reported metrics.
Thus the conformal step is post-hoc calibration rather than gradient training, and no model weights, covariance parameters, conjugate statistics, or hyperparameters are updated on the final evaluation windows.

\paragraph{Metrics.}
We report mean per-joint position error (MPJPE), final displacement error (FDE; here implemented as final-horizon MPJPE), and per-element Gaussian negative log-likelihood (NLL) (averaged over horizons and coordinates). For space, we present MPJPE/FDE as ``MPJPE / FDE'' in the main tables.
Because absolute metric scales and units vary across datasets, we additionally report per-dataset ranks in cross-dataset summaries; MPJPE+NLL rank denotes the mean of MPJPE and NLL ranks (Table~\ref{tab:overview_conjbayes_rank}).
For uncertainty, we evaluate nominal two-sided 95\% marginal intervals $\hat\mu\pm Z_{0.95}\hat\sigma$ (\textbf{Cov95/W95}) and split-conformalized intervals $\hat\mu\pm q_{t,j}\,\hat\sigma^{\mathrm{CP}}$ (\textbf{Cov95(\cp)/W95(\cp)}).
Here $Z_{0.95}=\Phi^{-1}(0.975)\approx 1.960$.
Cov95 is empirical marginal coverage (fraction of scalar targets $(t,c)$ covered), and W95 is the corresponding mean marginal interval width (averaged over $(t,c)$).
Conformal scales $\{q_{t,j}\}$ are calibrated \emph{per horizon and joint} by pooling standardized residuals $|y_{t,j,d}-\hat\mu_{t,j,d}|/\hat\sigma^{\mathrm{CP}}_{t,j,d}$ across the 3 coordinates $d$ of each joint on the calibration split.
Deterministic models are paired with a fixed isotropic observation noise $\sigma_{\mathrm{obs}}{=}0.017$\,m to compute NLL and nominal intervals; they do not output $\hat\sigma_{t,c}(x)$ and thus do not report conformal tube metrics.
The supplementary material additionally reports $\kappa$-binned Mondrian tubes and trajectory-level conformal ellipsoid metrics.

\paragraph{Training and calibration.}
All methods share the siMLPe/DCT backbone~\citep{guo2023simlpe} (MLP in DCT coefficient space, mapped to the time domain via iDCT; we forecast displacements relative to $x_T$ and add $x_T$ back at inference, as described in Sec.~\ref{sec:conjgraph}).
All methods are trained with Adam.
For the siMLPe backbone we use 2000 steps for base training and an 800-step mean fine-tune (MeanFT).
Probabilistic heads are trained for 800 steps with the mean frozen, using batch sizes 256/128/64 respectively.
For structured Gaussian heads, we initialize the mean from a trained MeanFT model and (when applicable) freeze the mean while optimizing only covariance parameters via NLL.
For \textbf{SPARC}, we fit the conjugate Bayesian last-layer statistics $\Lambda^{-1}_{n,t}$ in closed form on the training split, compute $\kappa_t(x)$, and tune the uncertainty/calibration hyperparameters per dataset on held-out training windows for the NLL--W95 trade-off.
These values are fixed for all reported evaluations.
Finally, we compute $\{q_{t,j}\}$ on the calibration split and evaluate calibrated tubes on the evaluation split.

\subsection{Baselines}
We compare against baseline families evaluated under the same protocol:
(i) deterministic point forecasters (siMLPe/MeanFT and recent variants such as AuxTasks~\citep{xu2023auxiliary_tasks}, Symplectic~\citep{chen2024symplectic}, DeFeeNet~\citep{sun2023defeenet}, and HumanMAC~\citep{chen2023humanmac}, plus graph-based SeSGCN~\citep{sampieri2022chico});
(ii) sampling/ensemble uncertainty baselines (deep ensembles~\citep{lakshminarayanan2017ensembles} and multi-head predictors such as GSPS~\citep{mao2021gsps});
(iii) calibration baselines (Bridge-CQR: conformalized quantile regression~\citep{romano2019cqr});
and (iv) Gaussian uncertainty heads (diagonal DCT-Bridge; separable matrix-normal~\citep{dawid1981matrixvariate}; and graph/GMRF-style structured covariances~\citep{besag1974spatial,rue2005gmrf});
and (v) external stochastic/diffusion forecasters (SkeletonDiffusion~\citep{curreli2025nonisotropic}, BeLFusion~\citep{barquero2023belfusion}, TransFusion~\citep{tian2024transfusion}, CoMusion~\citep{sun2024comusion}, SPARD~\citep{zhang2026spard}, MotionMap~\citep{hosseininejad2025motionmap}, and SLD-HMP~\citep{xu2024learning}).
When external methods provide public checkpoints, we initialize from the released weights and fine-tune under our protocol.
If checkpoints exist only for a subset of datasets, we initialize from the closest available checkpoint (e.g., Human3.6M/AMASS) and fine-tune.
Otherwise, we train the method from scratch.

\subsection{Cross-dataset comparison}
Because units differ across datasets, Table~\ref{tab:overview_conjbayes_rank} reports mean ranks ($\mathrm{MR}$; $\downarrow$ better) across the cross-dataset benchmark. Bold (underline) indicates best (second-best) among the shown methods (ties included); for coverage, best/second-best is defined by closeness to the 0.95 target.
Full per-dataset metrics are reported in the supplement.
\input{tables/overview_conjbayes_avg_rank.tex}
The rank table compares point accuracy (MPJPE/FDE), density quality (NLL), and calibrated tube efficiency (W95(\cp)) on a common scale.
SPARC has the leading density-quality and MPJPE+NLL ranks while retaining competitive point accuracy.
MeanFT remains a strong deterministic point predictor, but it does not provide calibrated uncertainty.
Sampling/generative baselines use multiple hypotheses or stochastic inference, while calibration-only baselines improve coverage without separating epistemic risk from structured aleatoric covariance.
SPARC combines competitive point accuracy ($\mathrm{MR}_{\mathrm{MPJPE}}=4.39$), rank-1 density quality ($\mathrm{MR}_{\mathrm{NLL}}=1.00$), rank-1 combined MPJPE+NLL trade-off ($2.69$), and efficient calibrated tubes ($\mathrm{MR}_{W95(\cp)}=2.50$) in one forward pass.
\paragraph{Practical benefits over existing methods.}
SPARC is not a uniform MPJPE winner, but it has the lowest NLL among shown methods on every dataset/protocol block while retaining competitive MPJPE; the supplementary detailed tables report the full metrics.
Its calibrated tubes remain efficient, so coverage is not obtained by indiscriminately widening intervals.
Beyond average metrics, the supplement shows that $\kappa$ stratifies error monotonically and supports selective deployment via abstention/fallback on high-$\kappa$ windows (thresholds tuned on held-out data, with no formal guarantee implied).
For robustness, the supplement re-evaluates fixed SPARC checkpoints over five split seeds, re-fits selected uncertainty checkpoints with three fit seeds, and repeats held-out Optuna-TPE HPO with three sampler seeds; coverage and tube width remain stable.

To isolate the contributions of structured covariance, epistemic scaling, and time-coupled $\kappa_t$ shrinkage, Fig.~\ref{fig:fc_out_ablation} summarizes an FC-Out component ablation across datasets; the supplement gives the per-dataset results.
Fig.~\ref{fig:fc_out_ablation} shows both the MPJPE/NLL component effects and a compact calibrated tube-efficiency rank; the supplementary tables give CP coverage/width for all FC-Out variants under the same fixed protocol.
\begin{figure}[htbp]
  \centering
  \begin{adjustbox}{max width=\linewidth}
    \input{figures/conjbayes_fc_out_ablation_tikz.tex}
  \end{adjustbox}
  \caption{FC-Out trade-off ablation. Left: active components. Center: within-family MPJPE+NLL-rank gain over base $\kappa(x)$; labels give absolute rank. Right: fixed-protocol $\mathrm{MR}_{W95(\mathrm{CP})}$ (lower is better). The center panel is within-family only and not comparable to Table~\ref{tab:overview_conjbayes_rank}.}
  \label{fig:fc_out_ablation}
\end{figure}
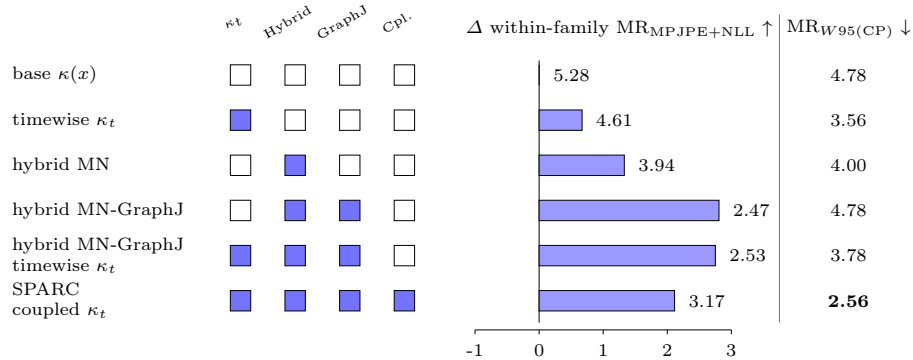

\subsection{FC-Out ablation: component contributions}
\paragraph{Ablation findings (all nine dataset/protocol blocks).}
Within this family, \emph{hybrid MN-GraphJ} gives the lowest MPJPE+NLL mean rank ($\mathrm{MR}_{\mathrm{MPJPE+NLL}}=2.47$), followed by uncoupled timewise $\kappa_t$ ($2.53$) and \emph{SPARC} (time-coupled $\kappa_t$; $3.17$) (Fig.~\ref{fig:fc_out_ablation}).
The difference reflects the metric being optimized.
Hybrid MN-GraphJ favors MPJPE+NLL, whereas the supplementary complete-CP comparison gives SPARC the lower tube-width rank ($\mathrm{MR}_{W95(\cp)}=2.56$ vs.\ $4.78$).
We therefore use \emph{SPARC} (time-coupled $\kappa_t$) as the default because it gives the best tube-efficiency trade-off for the calibrated deployment recipe, while the full benchmark against all baselines still ranks it first on global $\mathrm{MR}_{\mathrm{MPJPE+NLL}}$ (Table~\ref{tab:overview_conjbayes_rank}).
The main gains come from hybrid structured covariance and GraphJ coupling: relative to basic FC-Out $\kappa(x)$ (within-family $\mathrm{MR}_{\mathrm{MPJPE+NLL}}=5.28$), hybrid MN improves to $3.94$, and GraphJ improves further to $2.47$.
Time-coupled shrinkage should be read as a calibration-efficiency knob rather than a within-family NLL win.
Within this ablation it sacrifices some NLL rank ($\mathrm{MR}_{\mathrm{NLL}}=3.11$ vs.\ $2.44$ for hybrid MN-GraphJ without coupling) for the conformalized default used in the full benchmark (Table~\ref{tab:overview_conjbayes_rank}).
The DCT-pinv variant is omitted from Fig.~\ref{fig:fc_out_ablation} for space; the supplementary tables report it with overall rank $6.00$, so it is not part of the recommended recipe.

\paragraph{Compute.}
All structured heads (including SPARC) require one forward pass at inference.
SPARC adds only horizon-wise quadratic forms for $\kappa_t$ and does not require MC sampling.
	The closed-form fit of $\Lambda^{-1}_{n,t}$ is an offline pass over training features, and split conformal calibration uses only the $512$-window held-out calibration subset from the official evaluation protocol.
The supplement reports compute/memory overhead and latency measurements.

\paragraph{Controlled synthetic benchmark (supplementary).}
The controlled synthetic benchmark in the supplement tests component recovery under a known ground-truth uncertainty decomposition.

\subsection{Qualitative uncertainty visualization}

\begin{figure}[!htb]
  \centering
  \IfFileExists{figures/conjbayes_ci_example.pdf}{%
    \includegraphics[width=0.92\linewidth]{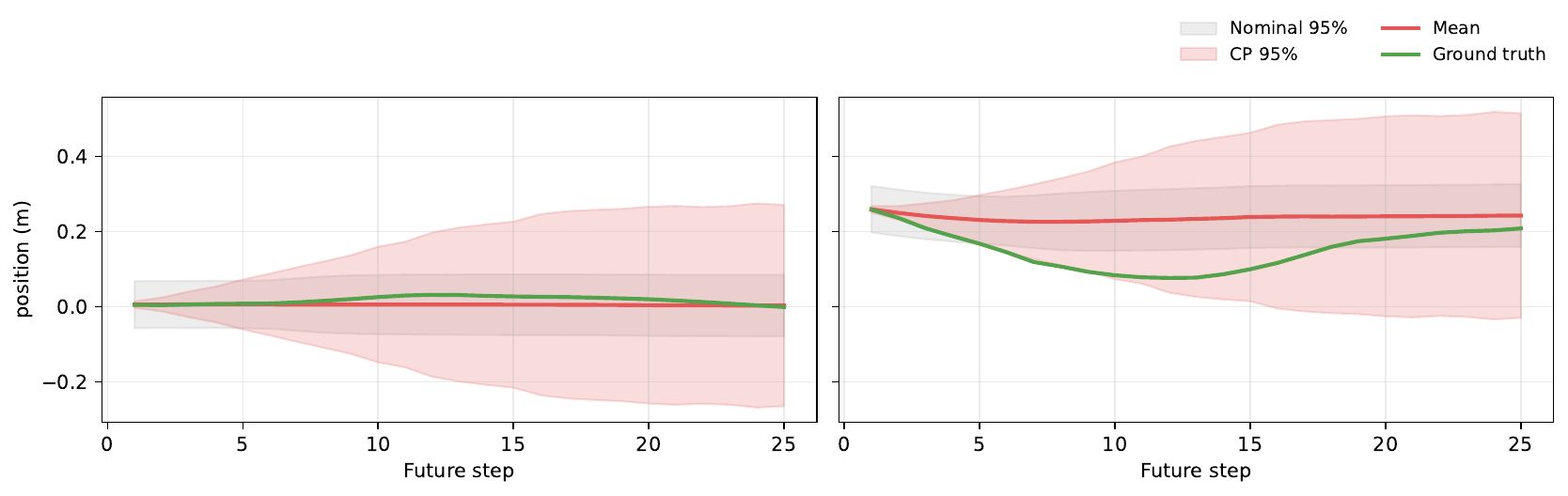}%
  }{%
    \fbox{\parbox{0.95\linewidth}{Missing figure: figures/conjbayes_ci_example.pdf}}%
  }
  \caption{Calibrated intervals for low vs.\ high mean epistemic scale $\bar{\kappa}$. Gray: nominal 95\% intervals ($\pm 1.960\,\sigma$). Red: split-conformal marginal intervals ($\pm q_{t,j}\,\sigma$), matching the target 95\% marginal coverage.}
  \label{fig:hyb_ci_example}
\end{figure}
\section{Discussion and limitations}
Risk-sensitive HRI, human-aware navigation, and safety filters need predictive envelopes for collision avoidance and personal-space constraints~\citep{lasota2017survey,kruse2013humanaware,sampieri2022chico}.
Overconfident tubes can leave insufficient safety margin under novel motions or shift, even when the mean forecast is accurate~\citep{kendall2017uncertainty,ovadia2019trust}.
SPARC addresses this with a single-pass epistemic risk score that inflates structured spatio-temporal covariance, while split conformal calibration restores target marginal coverage under exchangeability.
Practically, $\kappa$ can trigger conservative fallbacks using thresholds tuned on held-out calibration/validation data.

\paragraph{Limitations.}
The approach depends on the learned representation; poorly aligned features can weaken $\kappa$.
Conformal guarantees are marginal, and strong conditional guarantees remain open without additional assumptions~\citep{gibbs2023conditionalconformal}.
We report split-seed, fit-seed, and repeated-HPO diagnostics; hyperparameters are selected on held-out training windows and fixed before final calibration/evaluation.
Richer dependencies may require denser coordinate models at additional computational cost.

\section{Conclusion}
We introduced \emph{SPARC}, a hybrid uncertainty model combining conjugate Bayesian last-layer scaling, structured Gaussian aleatoric covariance, and conformal calibration.
It provides interpretable epistemic scaling, efficient inference, calibrated 95\% marginal prediction tubes, and low NLL across diverse motion forecasting datasets.
The full ablation supports \emph{SPARC} (hybrid MN-GraphJ with time-coupled $\kappa_t$ shrink) as the default, while FC-Out and DCT-pinv are analysis baselines.

\subsubsection*{Acknowledgments}
This work is funded by the federal state of Lower Saxony, Germany, as part of the project 76251-1337/2022 Kognitiv und Empathisch Intelligente Kobots -- ``KEIKO'' of the SPRUNG funding instrument.
We thank the anonymous reviewers for their constructive feedback.

%% file: tikz/arch_conjbayes_hybrid.tex
% SPARC formal model diagram: probabilistic core + post-hoc conformal layer.
\begin{tikzpicture}[
  font=\footnotesize,
  latent/.style={circle, draw=black!75, fill=orange!12, minimum size=8.4mm, inner sep=1.0pt, align=center},
  obs/.style={circle, draw=black!75, fill=blue!10, minimum size=8.4mm, inner sep=1.0pt, align=center},
  det/.style={rectangle, rounded corners=2.2pt, draw=black!75, fill=gray!10, inner sep=2.7pt, minimum height=7.2mm, align=center},
  param/.style={rectangle, rounded corners=2.2pt, draw=black!75, fill=green!10, inner sep=2.7pt, minimum height=7.2mm, align=center},
  cal/.style={rectangle, rounded corners=2.2pt, draw=black!75, fill=red!8, inner sep=2.7pt, minimum height=7.2mm, align=center},
  arrow/.style={->, >=Latex, semithick},
  note/.style={font=\scriptsize, text=black!70, align=center, fill=white, inner sep=1.2pt}
]

% -------------------- Panel A: probabilistic model --------------------
\node[obs]      (x)      at (0.4, 0.4)   {$x$};
\node[det]      (phi)    at (2.3, 0.4)   {$\phi_t(x)$};
\node[param]    (Lam)    at (3.8, 2.1)   {$\Lambda_t$};
\node[latent]   (B)      at (5.2, 2.1)   {$B_t$};
\node[det]      (mu)     at (5.2, 0.4)   {$\mu_t$};
\node[det]      (kappa)  at (4.4,-1.0)   {$\kappa_t$};
\node[det]      (ksh)    at (6.2,-1.0)   {$\tilde\kappa_t$};
\node[param]    (sstr)   at (4.4,-2.4)   {$\Sigma_{\mathrm{str},t}$};
\node[det]      (shyb)   at (6.2,-2.4)   {$\Sigma_{\mathrm{hyb},t}$};
\node[obs]      (y)      at (8.0, 0.4)   {$y_t$};
\node[det]      (stats)  at (10.2,0.4)   {$(\hat\mu_t,\hat\sigma_t)$};

\draw[arrow] (x) -- (phi);
\draw[arrow] (phi) -- (mu);
\draw[arrow] (Lam) -- (B);
\draw[arrow] (B) -- (mu);
\draw[arrow] (Lam.south) |- (kappa.north);
\draw[arrow] (phi.south) |- (kappa.west);
\draw[arrow] (kappa) -- (ksh);
\draw[arrow] (ksh) -- (shyb);
\draw[arrow] (sstr) -- (shyb);
\draw[arrow] (mu) -- (y);
\draw[arrow] (shyb.east) -| (y.south);
\draw[arrow] (y) -- (stats);

\node[note] at (4.35,-0.35) {$\kappa_t = 1 + \phi_t^\top \Lambda_t^{-1}\phi_t$};
\node[note] at (5.9,-2.95) {$\Sigma_{\mathrm{hyb},t}=\tilde\kappa_t\,\Sigma_{\mathrm{str},t}$};
\node[note] at (8.05,1.32) {$y_t \sim \mathcal N(\mu_t,\Sigma_{\mathrm{hyb},t})$};

\draw[rounded corners=3pt, draw=black!60] (1.6,-3.45) rectangle (12.15,2.7);
\node[font=\scriptsize\bfseries, anchor=south east, text=black!70] at (12.12,2.72) {A) Probabilistic core (for $t=1,\dots,H$)};

% -------------------- Panel B: conformal calibration --------------------
\node[param] (dcal)   at (2.4,-4.9) {$\mathcal D_{\mathrm{cal}}$};
\node[det]   (predin) at (10.2,-4.9) {Predictive summaries\\$(\hat\mu_{t,j},\hat\sigma_{t,j})$};
\node[cal, minimum width=45mm] (score) at (5.05,-6.7)
  {Nonconformity scores\\$s_{i,t,j,c}=\dfrac{|y_{i,t,j,c}-\hat\mu_{i,t,j,c}|}{\hat\sigma_{i,t,j,c}}$};
\node[cal, minimum width=30mm] (qt) at (9.45,-6.7)
  {Joint-wise quantiles\\$q_{t,j}=Q_{1-\alpha}(s_{\cdot,t,j,\cdot})$};
\node[cal, minimum width=44mm] (tube) at (13.75,-6.7)
  {Calibrated tubes\\$\widehat{\mathcal C}_{t,j}(x)=\hat\mu_{t,j}\pm q_{t,j}\hat\sigma_{t,j}$};

\draw[arrow] (stats.south) -- (predin.north);
\draw[arrow] (dcal.south) |- (score.north west);
\draw[arrow] (predin.south) |- (score.north east);
\draw[arrow] (score) -- (qt);
\draw[arrow] (qt) -- (tube);

\node[note, anchor=east] at (15.95,-7.95) {$\Pr\{y_{t,j}\in\widehat{\mathcal C}_{t,j}(X)\}\ge 1-\alpha$};
\node[note] at (4.15,-7.95) {$i$: calibration sample index,\quad $j$: joint};
\draw[rounded corners=3pt, draw=black!60] (1.6,-8.9) rectangle (16.8,-3.9);
\node[font=\scriptsize\bfseries, anchor=south east, text=black!70, fill=white, inner sep=1pt] at (16.77,-3.88) {B) Post-hoc split conformal layer};

\end{tikzpicture}

%% file: tables/overview_conjbayes_avg_rank.tex
\begin{table}[!htbp]
\centering
\caption{Mean rank (MR; $\downarrow$ better) across datasets. Ranks are computed per dataset and averaged over available datasets. $\mathrm{MR}_{\mathrm{MPJPE+NLL}}$ is the mean of MPJPE and NLL ranks (available for all models).}
\label{tab:overview_conjbayes_rank}
\scriptsize
\setlength{\tabcolsep}{5pt}
\renewcommand{\arraystretch}{1.10}
\begin{adjustbox}{max width=\linewidth}
\begin{tabular}{@{}lrrrr@{}}
\toprule
Model & $\mathrm{MR}_{\mathrm{MPJPE}}\downarrow$ & $\mathrm{MR}_{\mathrm{NLL}}\downarrow$ & $\mathrm{MR}_{W95(\mathrm{CP})}\downarrow$ & $\mathrm{MR}_{\mathrm{MPJPE+NLL}}\downarrow$ \\
\midrule
siMLPe (base) & 7.89 & 10.67 & -- & 9.28 \\
MeanFT (siMLPe fine-tune) & \textbf{2.22} & 6.11 & -- & 4.17 \\
AuxTasks (50/50) & 9.06 & 11.78 & -- & 10.42 \\
Symplectic (50/50) & 16.22 & 15.78 & -- & 16.00 \\
DeFeeNet (50/50) & 7.61 & 10.44 & -- & 9.03 \\
GSPS (50/50) & 7.67 & 9.00 & 5.44 & 8.33 \\
HumanMAC (50/50) & 10.22 & 10.00 & -- & 10.11 \\
SeSGCN (teacher) & 13.39 & 11.78 & -- & 12.58 \\
Deep ensemble & 8.22 & 5.56 & 5.39 & 6.89 \\
Bridge-CQR & 4.61 & \underline{3.00} & 9.11 & \underline{3.81} \\
SkeletonDiffusion (CVPR 2025)~\citep{curreli2025nonisotropic} & 15.00 & 13.67 & 9.78 & 14.33 \\
BeLFusion (ICCV 2023)~\citep{barquero2023belfusion} & 9.56 & 8.00 & 5.89 & 8.78 \\
TransFusion (RA-L 2024)~\citep{tian2024transfusion} & 12.44 & 9.78 & 4.78 & 11.11 \\
CoMusion (ECCV 2024)~\citep{sun2024comusion} & 6.06 & 6.44 & 3.78 & 6.25 \\
SPARD (AAAI 2026)~\citep{zhang2026spard} & 7.11 & 6.56 & \underline{3.39} & 6.83 \\
MotionMap (CVPR 2025)~\citep{hosseininejad2025motionmap} & 16.11 & 16.22 & 9.44 & 16.17 \\
SLD-HMP (ECCV 2024)~\citep{xu2024learning} & 13.22 & 15.22 & 6.50 & 14.22 \\
SPARC$^{\dagger}$ & \underline{4.39} & \textbf{1.00} & \textbf{2.50} & \textbf{2.69} \\
\bottomrule
\end{tabular}
\end{adjustbox}
\par\smallskip
{\footnotesize\textbf{Legend:} \textbf{bold} = best; \underline{underline} = second-best (ties included).}
\end{table}

%% file: figures/conjbayes_fc_out_ablation_tikz.tex
\begin{tikzpicture}[x=1cm,y=1cm, font=\scriptsize]
\node[anchor=south, rotate=30, font=\tiny] at (3.25,0.55) {$\kappa_t$};
\node[anchor=south, rotate=30, font=\tiny] at (4.00,0.55) {Hybrid};
\node[anchor=south, rotate=30, font=\tiny] at (4.75,0.55) {GraphJ};
\node[anchor=south, rotate=30, font=\tiny] at (5.50,0.55) {Cpl.};
\node[anchor=center] at (8.45,0.63) {$\Delta$ within-family $\mathrm{MR}_{\mathrm{MPJPE+NLL}}\uparrow$};
\node[anchor=center] at (11.60,0.63) {$\mathrm{MR}_{W95(\mathrm{CP})}\downarrow$};
\draw[black, line width=0.3pt] (7.35,0.20) -- (7.35,-3.40);
\draw[black, line width=0.3pt] (6.47,-3.50) -- (6.47,-3.56);
\node[anchor=north] at (6.47,-3.58) {-1};
\draw[black, line width=0.3pt] (7.35,-3.50) -- (7.35,-3.56);
\node[anchor=north] at (7.35,-3.58) {0};
\draw[black, line width=0.3pt] (8.23,-3.50) -- (8.23,-3.56);
\node[anchor=north] at (8.23,-3.58) {1};
\draw[black, line width=0.3pt] (9.11,-3.50) -- (9.11,-3.56);
\node[anchor=north] at (9.11,-3.58) {2};
\draw[black, line width=0.3pt] (9.99,-3.50) -- (9.99,-3.56);
\node[anchor=north] at (9.99,-3.58) {3};
\draw[black, line width=0.3pt] (6.47,-3.50) -- (9.99,-3.50);
\draw[black!55, line width=0.2pt] (10.65,0.83) -- (10.65,-3.40);
\node[anchor=west, align=left, text width=3.05cm] at (0.00,0.00) {base $\kappa(x)$};
\path[draw=black, line width=0.2pt] (3.11,-0.14) rectangle (3.39,0.14);
\path[draw=black, line width=0.2pt] (3.86,-0.14) rectangle (4.14,0.14);
\path[draw=black, line width=0.2pt] (4.61,-0.14) rectangle (4.89,0.14);
\path[draw=black, line width=0.2pt] (5.36,-0.14) rectangle (5.64,0.14);
\path[draw=black, line width=0.2pt, fill=black!15] (7.35,-0.14) rectangle (7.35,0.14);
\node[anchor=west] at (7.43,0.00) {5.28};
\node[anchor=center] at (11.60,0.00) {4.78};
\node[anchor=west, align=left, text width=3.05cm] at (0.00,-0.62) {timewise $\kappa_t$};
\path[draw=black, line width=0.2pt, fill=blue!55] (3.11,-0.76) rectangle (3.39,-0.48);
\path[draw=black, line width=0.2pt] (3.86,-0.76) rectangle (4.14,-0.48);
\path[draw=black, line width=0.2pt] (4.61,-0.76) rectangle (4.89,-0.48);
\path[draw=black, line width=0.2pt] (5.36,-0.76) rectangle (5.64,-0.48);
\path[draw=black, line width=0.2pt, fill=blue!40] (7.35,-0.76) rectangle (7.94,-0.48);
\node[anchor=west] at (8.02,-0.62) {4.61};
\node[anchor=center] at (11.60,-0.62) {3.56};
\node[anchor=west, align=left, text width=3.05cm] at (0.00,-1.24) {hybrid MN};
\path[draw=black, line width=0.2pt] (3.11,-1.38) rectangle (3.39,-1.10);
\path[draw=black, line width=0.2pt, fill=blue!55] (3.86,-1.38) rectangle (4.14,-1.10);
\path[draw=black, line width=0.2pt] (4.61,-1.38) rectangle (4.89,-1.10);
\path[draw=black, line width=0.2pt] (5.36,-1.38) rectangle (5.64,-1.10);
\path[draw=black, line width=0.2pt, fill=blue!40] (7.35,-1.38) rectangle (8.52,-1.10);
\node[anchor=west] at (8.60,-1.24) {3.94};
\node[anchor=center] at (11.60,-1.24) {4.00};
\node[anchor=west, align=left, text width=3.05cm] at (0.00,-1.86) {hybrid MN-GraphJ};
\path[draw=black, line width=0.2pt] (3.11,-2.00) rectangle (3.39,-1.72);
\path[draw=black, line width=0.2pt, fill=blue!55] (3.86,-2.00) rectangle (4.14,-1.72);
\path[draw=black, line width=0.2pt, fill=blue!55] (4.61,-2.00) rectangle (4.89,-1.72);
\path[draw=black, line width=0.2pt] (5.36,-2.00) rectangle (5.64,-1.72);
\path[draw=black, line width=0.2pt, fill=blue!40] (7.35,-2.00) rectangle (9.82,-1.72);
\node[anchor=west] at (9.90,-1.86) {2.47};
\node[anchor=center] at (11.60,-1.86) {4.78};
\node[anchor=west, align=left, text width=3.05cm] at (0.00,-2.48) {hybrid MN-GraphJ\\timewise $\kappa_t$};
\path[draw=black, line width=0.2pt, fill=blue!55] (3.11,-2.62) rectangle (3.39,-2.34);
\path[draw=black, line width=0.2pt, fill=blue!55] (3.86,-2.62) rectangle (4.14,-2.34);
\path[draw=black, line width=0.2pt, fill=blue!55] (4.61,-2.62) rectangle (4.89,-2.34);
\path[draw=black, line width=0.2pt] (5.36,-2.62) rectangle (5.64,-2.34);
\path[draw=black, line width=0.2pt, fill=blue!40] (7.35,-2.62) rectangle (9.77,-2.34);
\node[anchor=west] at (9.85,-2.48) {2.53};
\node[anchor=center] at (11.60,-2.48) {3.78};
\node[anchor=west, align=left, text width=3.05cm] at (0.00,-3.10) {SPARC\\coupled $\kappa_t$};
\path[draw=black, line width=0.2pt, fill=blue!55] (3.11,-3.24) rectangle (3.39,-2.96);
\path[draw=black, line width=0.2pt, fill=blue!55] (3.86,-3.24) rectangle (4.14,-2.96);
\path[draw=black, line width=0.2pt, fill=blue!55] (4.61,-3.24) rectangle (4.89,-2.96);
\path[draw=black, line width=0.2pt, fill=blue!55] (5.36,-3.24) rectangle (5.64,-2.96);
\path[draw=black, line width=0.2pt, fill=blue!40] (7.35,-3.24) rectangle (9.21,-2.96);
\node[anchor=west] at (9.29,-3.10) {3.17};
\node[anchor=center] at (11.60,-3.10) {\textbf{2.56}};
\end{tikzpicture}

%% file: paper_supp.tex
\appendix
\section{Human3.6M detailed results}
\label{sec:app_h36m}
\begingroup
\centering
\captionof{table}{Human3.6M (H=25; seed=304, $n_{\mathrm{cal}}{=}512$, $n_{\mathrm{eval}}{=}1024$, $\alpha{=}0.05$): component isolation with a fixed mean predictor. MatrixNormal, MN-GraphJ, and SPARC share the same MeanFT mean (freeze-mean), so differences reflect uncertainty modeling; conformal calibration (\cp) moves coverage toward the 0.95 target with modest width change.}
\label{tab:h36m_hyb_main}
\scriptsize
\setlength{\tabcolsep}{4pt}
\resizebox{\linewidth}{!}{%
\begin{tabular}{@{}lrrrrrr@{}}
  \toprule
  Model & MPJPE$\downarrow$ & NLL$\downarrow$ & Cov95 & W95$\downarrow$ & Cov95(\cp) & W95(\cp)$\downarrow$ \\
  \midrule
	MeanFT (deterministic) + fixed $\sigma_{\mathrm{obs}}$ & 0.0721 & 5.27 & 0.715 & 0.067 & -- & -- \\
	DCT-Bridge (diagonal Gaussian) & 0.0721 & -1.03 & 0.903 & 0.184 & 0.945 & 0.257 \\
	MatrixNormal (freeze-mean) & 0.0721 & -1.26 & 0.925 & 0.216 & 0.945 & 0.253 \\
	MatrixNormal-GraphJ (freeze-mean) & 0.0721 & -1.66 & 0.868 & 0.137 & 0.945 & 0.251 \\
	\textbf{SPARC$^{\dagger}$} & 0.0721 & -1.66 & 0.868 & 0.137 & 0.946 & 0.248 \\
  \bottomrule
\end{tabular}}
\par\smallskip
{\footnotesize $^{\dagger}$ (MN-GraphJ, time-coupled $\kappa_t$, shrink).}
\endgroup

\section{Split-seed robustness}
\label{sec:app_split_seed}
To check that the calibrated tube metrics are not an artifact of the single fixed calibration/evaluation permutation used in the main benchmark, we re-evaluate fixed SPARC checkpoints over multiple split seeds.
Only the held-out evaluation-window permutation changes; trained weights, covariance parameters, conjugate statistics, and tuned hyperparameters remain fixed.
\input{tables/conjgraph_cp_split_seed_stability.tex}

We repeat the full held-out HPO on three representative datasets with different Optuna-TPE seeds to test whether the tuned recipe is tied to a single Optuna trajectory.
Checkpoint inputs, budget, held-out objective, and the five-parameter search space ($\lambda_0$, $\rho$, $\gamma$, $\lambda_{0,\mathrm{joint}}$, $\gamma_{\mathrm{joint}}$) are unchanged.
\input{tables/conjgraph_cp_hpo_seed_stability.tex}

Finally, we isolate fit-seed variability after HPO.
We keep the selected full-budget HPO hyperparameters, input checkpoints, and calibration/evaluation protocol fixed, and re-fit the SPARC uncertainty checkpoint with different RNG seeds.
\input{tables/conjgraph_cp_fit_seed_stability.tex}

\section{Cross-dataset detailed results}
\label{sec:app_cross_dataset}
\input{tables/overview_conjbayes_all_datasets.tex}

\section{Full FC-Out ablation (per-dataset)}
\label{sec:app_ablation_full}
These tables isolate component effects on point accuracy and Gaussian density quality.
CP coverage/width is part of the complete calibrated deployment recipe and is therefore reported for every FC-Out ablation variant under the same fixed protocol.
This makes the MPJPE/NLL trade-off and calibrated tube-efficiency trade-off directly comparable.
\input{tables/ablation_conjbayes_fc_out_components.tex}
\input{tables/ablation_conjbayes_fc_out_avg_rank.tex}
\FloatBarrier
\input{tables/ablation_conjbayes_fc_out_all_datasets.tex}

\section{Epistemic scale diagnostics}
\label{sec:app_kappa_diag}
The diagnostics focus on three decisions: uncertainty inflation, error stratification by $\kappa$, and post-hoc selective deployment.
Fig.~\ref{fig:hyb_kappa_diag} shows how epistemic scaling inflates uncertainty and whether $\kappa$ stratifies error monotonically; Table~\ref{tab:conjgraph_variants_all} reports the conformal variants enabled by $\kappa$ and the structured covariance; Table~\ref{tab:kappa_selective} gives the selective-deployment analysis used only as a diagnostic.

\begin{figure}[t]
  \centering
  \begin{minipage}[t]{0.49\linewidth}
    \centering
    \IfFileExists{figures/conjbayes_uncertainty_inflation.pdf}{%
      \includegraphics[width=\linewidth]{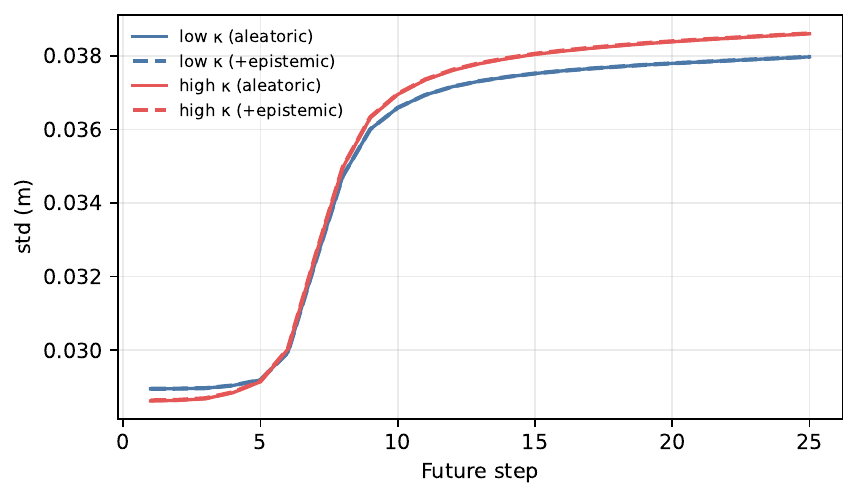}%
    }{%
      \fbox{\parbox{0.95\linewidth}{Missing figure: figures/conjbayes_uncertainty_inflation.pdf}}%
    }
  \end{minipage}\hfill
  \begin{minipage}[t]{0.49\linewidth}
    \centering
    \IfFileExists{figures/conjbayes_kappa_binned_error.pdf}{%
      \includegraphics[width=\linewidth]{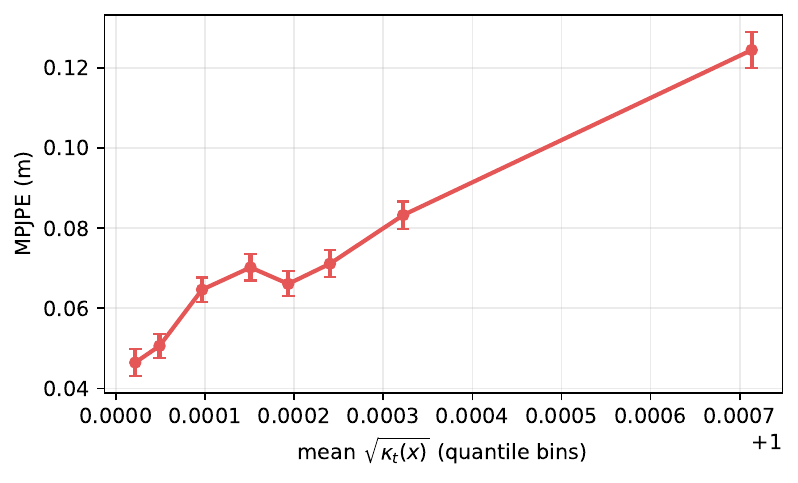}%
    }{%
      \fbox{\parbox{0.95\linewidth}{Missing figure: figures/conjbayes_kappa_binned_error.pdf}}%
    }
  \end{minipage}
  \caption{Epistemic scale diagnostics. \textbf{(a)} $\kappa_t$ inflates the structured marginal standard deviation for low- vs.\ high-$\kappa$ examples. \textbf{(b)} MPJPE increases across bins of mean $\sqrt{\kappa_t(x)}$ (error bars: standard error), supporting $\kappa$ as a risk stratifier beyond average coverage metrics.}
  \label{fig:hyb_kappa_diag}
\end{figure}

\input{tables/conjgraph_variants_all_datasets.tex}
\input{tables/conjbayes_kappa_selective_h36m.tex}
\input{tables/h36m_conjgraph_kappa_mondrian.tex}
\input{tables/h36m_conjgraph_traj_set.tex}

\section{Controlled synthetic uncertainty benchmark}
\label{sec:app_synth}
To complement real-dataset evaluation, we run a controlled synthetic benchmark with known uncertainty structure and test whether SPARC recovers the intended epistemic/aleatoric components.
Given observed prefixes, we generate future trajectories as
\[
y_t = \mu_t^\star(x) + \varepsilon_t,\qquad
\varepsilon \sim \mathcal{N}\!\bigl(0,\ \kappa_t^\star(x)\,\Sigma_T^\star \otimes \Sigma_C^\star\bigr),
\]
where $\Sigma_T^\star$ is first-order autoregressive (AR(1)) temporal covariance, $\Sigma_C^\star$ is graph-coupled coordinate covariance, and $\kappa_t^\star(x)$ is an input-dependent heteroscedastic scale.
We use a protocol-aligned split ($n_{\mathrm{train}}{=}16{,}000$, $n_{\mathrm{cal}}{=}512$, $n_{\mathrm{eval}}{=}1024$) plus a shifted evaluation split.
Figure~\ref{fig:synth_dataset_overview} first summarizes what the synthetic benchmark encodes: in-distribution vs.\ shifted trajectory families, the induced uncertainty-regime shift in $\sqrt{\kappa_t^\star}$, and horizon-wise uncertainty profiles.

\input{tables/synthetic_conjgraph_tuning_summary.tex}

Table~\ref{tab:synth_tuning_summary} compares the legacy synthetic fit to our tuned setup (feature-standardized Bayesian precision fit + calibrated epistemic sensitivity scale).
		The tuned setup improves epistemic tracking ($r_\kappa$: $0.475\!\rightarrow\!0.499$ in-distribution; $0.820\!\rightarrow\!0.848$ under shift), where $r_\kappa$ is the Pearson correlation between the predicted and true sequence-level mean $\sqrt{\kappa_t^\star}$.
It maintains temporal component recovery (temporal-correlation mean absolute error (MAE) $0.022\!\rightarrow\!0.022$), improves joint-correlation MAE ($0.103\!\rightarrow\!0.100$), and improves shifted conformal coverage ($0.667\!\rightarrow\!0.847$) while keeping in-distribution Cov95(\cp) close to target ($\approx0.949$).
Figure~\ref{fig:synth_conjgraph_recovery} visualizes these recovered components: panels (a)--(b) preserve the AR-style temporal correlation geometry, panel (c) shows tighter true-vs.-learned $\sqrt{\kappa}$ alignment (with stronger separation under shift), and panels (d)--(f) show improved graph-coupled correlation recovery and risk-profile matching over horizon.
\paragraph{$\kappa$-conditioned (Mondrian) conformal tubes.}
Beyond global split conformal, we consider a $\kappa$-conditioned (Mondrian) variant that bins sequences by mean $\sqrt{\kappa_t(x)}$ and calibrates group-wise quantiles, yielding risk-adaptive tube scaling.
Table~\ref{tab:synth_kappa_mondrian} reports Cov95/W95 by $\kappa$-bin for both unconditioned \cp and the Mondrian variant.
\input{tables/synthetic_conjgraph_kappa_mondrian.tex}

\paragraph{Trajectory-level conformal sets.}
We also construct a trajectory-level conformal ellipsoid using a structured Mahalanobis/whitened-residual score under $\Sigma_T^{\mathrm{hyb}}\otimes\Sigma_C$ as an alternative to axis-aligned tubes.
Table~\ref{tab:synth_traj_set} compares coverage and set volume against marginal \cp rectangles.
\input{tables/synthetic_conjgraph_traj_set.tex}
\begin{figure}[t]
  \centering
  \begin{minipage}[t]{0.49\linewidth}
    \centering
    \IfFileExists{figures/synthetic_conjgraph_dataset_overview_a_id_traj.pdf}{%
      \includegraphics[width=\linewidth]{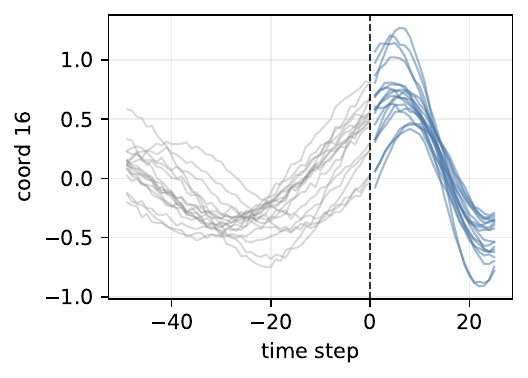}%
    }{%
      \fbox{\parbox{0.95\linewidth}{Missing figure: figures/synthetic_conjgraph_dataset_overview_a_id_traj.pdf}}%
    }
  \end{minipage}\hfill
  \begin{minipage}[t]{0.49\linewidth}
    \centering
    \IfFileExists{figures/synthetic_conjgraph_dataset_overview_b_shift_traj.pdf}{%
      \includegraphics[width=\linewidth]{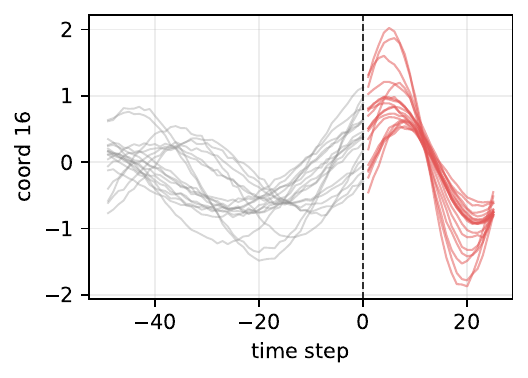}%
    }{%
      \fbox{\parbox{0.95\linewidth}{Missing figure: figures/synthetic_conjgraph_dataset_overview_b_shift_traj.pdf}}%
    }
  \end{minipage}

  \vspace{1.5mm}

  \begin{minipage}[t]{0.49\linewidth}
    \centering
    \IfFileExists{figures/synthetic_conjgraph_dataset_overview_c_kappa_hist.pdf}{%
      \includegraphics[width=\linewidth]{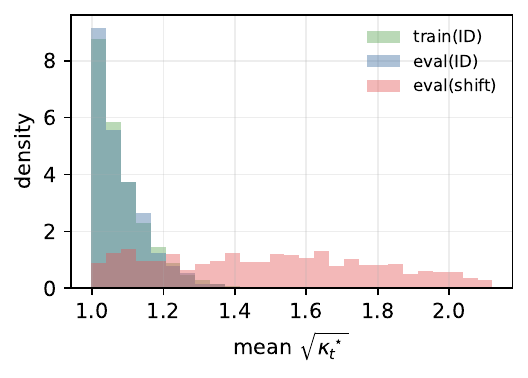}%
    }{%
      \fbox{\parbox{0.95\linewidth}{Missing figure: figures/synthetic_conjgraph_dataset_overview_c_kappa_hist.pdf}}%
    }
  \end{minipage}\hfill
  \begin{minipage}[t]{0.49\linewidth}
    \centering
    \IfFileExists{figures/synthetic_conjgraph_dataset_overview_d_kappa_profile.pdf}{%
      \includegraphics[width=\linewidth]{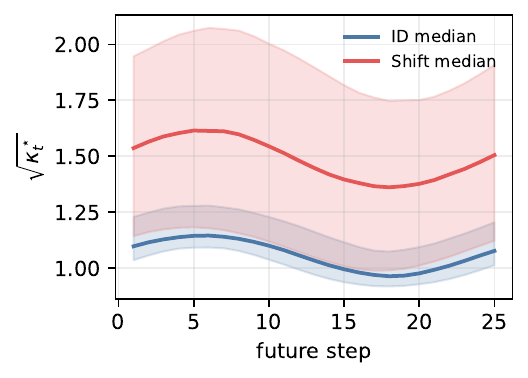}%
    }{%
      \fbox{\parbox{0.95\linewidth}{Missing figure: figures/synthetic_conjgraph_dataset_overview_d_kappa_profile.pdf}}%
    }
  \end{minipage}
  \caption{Synthetic dataset overview used for SPARC stress testing. \textbf{(a)} Representative in-distribution trajectory slices (same coordinate) with observed prefix and future continuation. \textbf{(b)} Representative shifted trajectory slices. \textbf{(c)} Distribution shift of the true mean epistemic scale $\sqrt{\kappa_t^\star}$. \textbf{(d)} Horizon-wise 10--50--90\% profiles of $\sqrt{\kappa_t^\star}$ for in-distribution vs.\ shifted splits.}
  \label{fig:synth_dataset_overview}
\end{figure}
\begin{figure}[t]
	  \centering
	  \begin{minipage}[t]{0.32\linewidth}
	    \centering
	    \IfFileExists{figures/synthetic_conjgraph_recovery_a_corr_t_true.pdf}{%
	      \includegraphics[width=\linewidth]{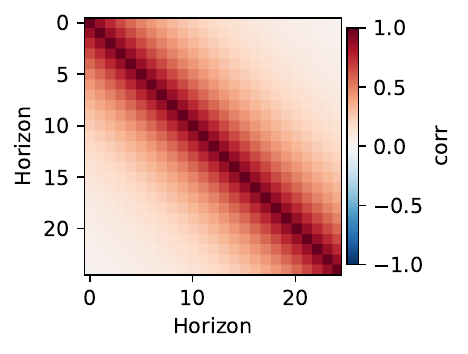}%
	    }{%
	      \fbox{\parbox{0.95\linewidth}{Missing figure: figures/synthetic_conjgraph_recovery_a_corr_t_true.pdf}}%
	    }
	  \end{minipage}\hfill
	  \begin{minipage}[t]{0.32\linewidth}
	    \centering
	    \IfFileExists{figures/synthetic_conjgraph_recovery_b_corr_t_learned.pdf}{%
	      \includegraphics[width=\linewidth]{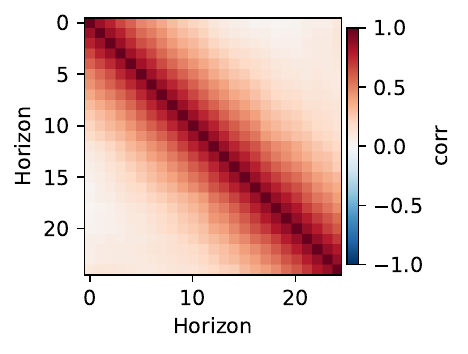}%
	    }{%
	      \fbox{\parbox{0.95\linewidth}{Missing figure: figures/synthetic_conjgraph_recovery_b_corr_t_learned.pdf}}%
	    }
	  \end{minipage}\hfill
	  \begin{minipage}[t]{0.32\linewidth}
	    \centering
	    \IfFileExists{figures/synthetic_conjgraph_recovery_c_kappa_scatter.pdf}{%
	      \includegraphics[width=\linewidth]{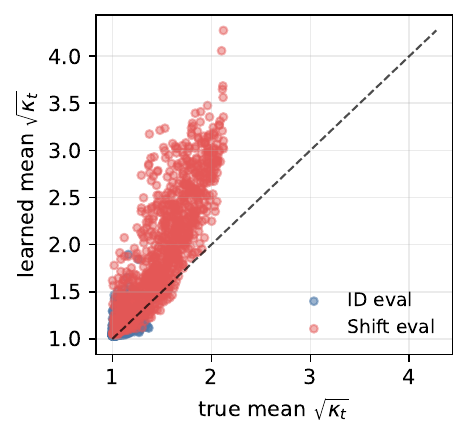}%
	    }{%
	      \fbox{\parbox{0.95\linewidth}{Missing figure: figures/synthetic_conjgraph_recovery_c_kappa_scatter.pdf}}%
	    }
	  \end{minipage}

	  \vspace{1.5mm}

	  \begin{minipage}[t]{0.32\linewidth}
	    \centering
	    \IfFileExists{figures/synthetic_conjgraph_recovery_d_corr_j_true.pdf}{%
	      \includegraphics[width=\linewidth]{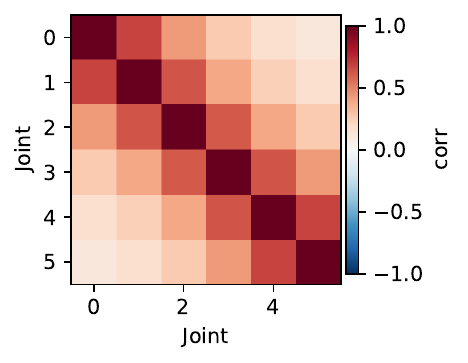}%
	    }{%
	      \fbox{\parbox{0.95\linewidth}{Missing figure: figures/synthetic_conjgraph_recovery_d_corr_j_true.pdf}}%
	    }
	  \end{minipage}\hfill
	  \begin{minipage}[t]{0.32\linewidth}
	    \centering
	    \IfFileExists{figures/synthetic_conjgraph_recovery_e_corr_j_learned.pdf}{%
	      \includegraphics[width=\linewidth]{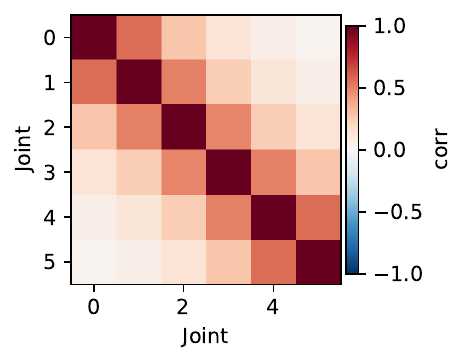}%
	    }{%
	      \fbox{\parbox{0.95\linewidth}{Missing figure: figures/synthetic_conjgraph_recovery_e_corr_j_learned.pdf}}%
	    }
	  \end{minipage}\hfill
	  \begin{minipage}[t]{0.32\linewidth}
	    \centering
	    \IfFileExists{figures/synthetic_conjgraph_recovery_f_risk_profile.pdf}{%
	      \includegraphics[width=\linewidth]{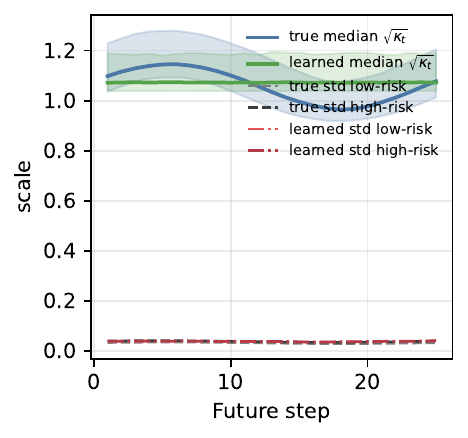}%
	    }{%
	      \fbox{\parbox{0.95\linewidth}{Missing figure: figures/synthetic_conjgraph_recovery_f_risk_profile.pdf}}%
	    }
	  \end{minipage}

	  \caption{Synthetic component-recovery benchmark for SPARC. \textbf{(a)} True temporal correlation. \textbf{(b)} Learned temporal correlation. \textbf{(c)} True vs.\ learned epistemic scale ($\sqrt{\kappa}$) on in-distribution and shifted splits. \textbf{(d)} True graph-coupled joint correlation. \textbf{(e)} Learned joint correlation. \textbf{(f)} Risk-profile recovery over horizon.}
	  \label{fig:synth_conjgraph_recovery}
	\end{figure}
The synthetic stress test indicates that SPARC can (i) preserve and recover structured spatio-temporal correlations while (ii) tracking an input-dependent epistemic scale that responds under shift. The conformal variants show how $\kappa$ and the structured covariance support risk-adaptive calibration and coherent trajectory-level sets.

\section{Compute and storage overhead}
\label{sec:app_compute}
To contextualize the practical cost of uncertainty layers, Table~\ref{tab:appendix_compute_cost} reports parameter counts, additional stored non-parameter state (for the analytic $\kappa$ statistics), and single-sample inference latency on a representative Human3.6M test sequence (batch size 1, horizon $H{=}25$).
Absolute times depend on hardware and implementation details; the intent is to quantify the order-of-magnitude overhead relative to the deterministic backbone and to contrast analytic SPARC against sampling-based epistemic baselines.
For stochastic baselines we measure end-to-end cost to form predictive mean/std from $K$ samples (reported in the table); for the heaviest external baselines we only report GPU timings.
\input{tables/runtime_structured.tex}
\input{tables/appendix_compute_cost.tex}

%% file: tables/conjgraph_cp_split_seed_stability.tex
\begin{table}[t]
\centering
\caption{Split-seed robustness of SPARC. We keep the trained checkpoint and all tuned hyperparameters fixed, vary only the RNG seed used to split held-out evaluation windows into calibration and evaluation subsets, and report mean $\pm$ standard deviation over split seeds. Calibration computes only split-conformal quantiles; no weights, covariance parameters, conjugate statistics, or hyperparameters are updated.}
\label{tab:conjgraph_cp_split_seed_stability}
\scriptsize
\setlength{\tabcolsep}{4pt}
\begin{tabular}{@{}lrrrrrr@{}}
\toprule
Dataset & Seeds & $n_{\mathrm{cal}}$ & $n_{\mathrm{eval}}$ & MPJPE$\downarrow$ & NLL$\downarrow$ & Cov95(\cp) / W95(\cp)$\downarrow$ \\
\midrule
H36M & 5 & 512 & 1024 & 0.0710 $\pm$ 0.0010 & -1.67 $\pm$ 0.01 & 0.950 $\pm$ 0.004 / 0.258 $\pm$ 0.007 \\
AMASS & 5 & 512 & 1024 & 0.0584 $\pm$ 0.0012 & -2.37 $\pm$ 0.06 & 0.953 $\pm$ 0.004 / 0.212 $\pm$ 0.005 \\
AnDy & 5 & 512 & 1024 & 0.0310 $\pm$ 0.0011 & -3.54 $\pm$ 0.02 & 0.951 $\pm$ 0.004 / 0.139 $\pm$ 0.004 \\
CHICO & 5 & 512 & 1024 & 0.0446 $\pm$ 0.0006 & -2.19 $\pm$ 0.01 & 0.951 $\pm$ 0.002 / 0.135 $\pm$ 0.003 \\
CMU & 5 & 512 & 320 & 2.7695 $\pm$ 0.0636 & 10.14 $\pm$ 0.79 & 0.950 $\pm$ 0.005 / 8.479 $\pm$ 0.228 \\
HA4M & 5 & 13 & 13 & 0.1214 $\pm$ 0.0016 & -2.28 $\pm$ 0.01 & 0.974 $\pm$ 0.008 / 0.256 $\pm$ 0.002 \\
LaFAN1 & 5 & 512 & 1024 & 0.5437 $\pm$ 0.0069 & -0.83 $\pm$ 0.01 & 0.951 $\pm$ 0.002 / 1.568 $\pm$ 0.011 \\
LaFAN1-HF & 5 & 512 & 1024 & 0.3166 $\pm$ 0.0069 & -0.05 $\pm$ 0.01 & 0.950 $\pm$ 0.003 / 0.939 $\pm$ 0.022 \\
3DPW & 5 & 512 & 1024 & 0.0583 $\pm$ 0.0012 & -1.92 $\pm$ 0.05 & 0.950 $\pm$ 0.003 / 0.192 $\pm$ 0.003 \\
\bottomrule
\end{tabular}
\par\smallskip{\footnotesize Split seeds: 101, 202, 304, 405, 506.}
\end{table}

%% file: tables/conjgraph_cp_hpo_seed_stability.tex
\begin{table}[t]
\centering
\caption{HPO-seed robustness of SPARC. We repeat held-out Optuna-TPE HPO over $\lambda_0$, $\rho$, $\gamma$, $\lambda_{0,\mathrm{joint}}$, and $\gamma_{\mathrm{joint}}$, and report mean $\pm$ standard deviation over HPO seeds. Seed 304 is the submitted full-budget run; the additional seeds use the same checkpoint inputs, search space, budget, and held-out objective.}
\label{tab:conjgraph_cp_hpo_seed_stability}
\scriptsize
\setlength{\tabcolsep}{4pt}
\resizebox{\linewidth}{!}{%
\begin{tabular}{@{}lrrrrr@{}}
\toprule
Dataset & HPO seeds & Trials/seed & NLL$\downarrow$ & Cov95(\cp) & W95(\cp)$\downarrow$ \\
\midrule
H36M & 101,304,506 & 96 & -1.664999 $\pm$ 0.000037 & 0.945531 $\pm$ 0.000002 & 0.248099 $\pm$ 0.000001 \\
CHICO & 101,304,506 & 96 & -2.207060 $\pm$ 0.000003 & 0.948860 $\pm$ 0.000000 & 0.130802 $\pm$ 0.000000 \\
3DPW & 101,304,506 & 96 & -1.943978 $\pm$ 0.000014 & 0.946786 $\pm$ 0.000000 & 0.187710 $\pm$ 0.000001 \\
\bottomrule
\end{tabular}
}
\end{table}

%% file: tables/conjgraph_cp_fit_seed_stability.tex
\begin{table}[t]
\centering
\caption{Fit-seed robustness of SPARC with fixed HPO hyperparameters. For each dataset we keep the selected full-budget HPO hyperparameters, input checkpoints, and calibration/evaluation protocol fixed, and re-fit the SPARC uncertainty checkpoint with different RNG seeds.}
\label{tab:conjgraph_cp_fit_seed_stability}
\scriptsize
\setlength{\tabcolsep}{4pt}
\resizebox{\linewidth}{!}{%
\begin{tabular}{@{}lrrrrr@{}}
\toprule
Dataset & Fit seeds & NLL$\downarrow$ & Cov95(\cp) & W95(\cp)$\downarrow$ & MPJPE$\downarrow$ \\
\midrule
H36M & 101,304,506 & -1.664890 $\pm$ 0.000007 & 0.945531 $\pm$ 0.000004 & 0.248107 $\pm$ 0.000001 & 0.072110 $\pm$ 0.000000 \\
CHICO & 101,304,506 & -2.207059 $\pm$ 0.000004 & 0.948859 $\pm$ 0.000002 & 0.130801 $\pm$ 0.000000 & 0.044460 $\pm$ 0.000000 \\
3DPW & 101,304,506 & -1.943956 $\pm$ 0.000005 & 0.946787 $\pm$ 0.000003 & 0.187711 $\pm$ 0.000000 & 0.058032 $\pm$ 0.000000 \\
\bottomrule
\end{tabular}
}
\end{table}

%% file: tables/overview_conjbayes_all_datasets.tex
\begingroup
\tiny
\setlength{\tabcolsep}{1.2pt}
\renewcommand{\arraystretch}{1.06}

\refstepcounter{table}
\noindent\textbf{Table \thetable:} Baselines vs.\ SPARC$^{\dagger}$ across datasets (H=25). Metrics: MPJPE/FDE, NLL, and split-conformal 95\% tube coverage/width (Cov95(CP)/W95(CP)); for Cov95(CP), ranking is by closeness to 0.95.\label{tab:overview_conjbayes_all}
\par\medskip

\tablefirsthead{%
\toprule
& \multicolumn{4}{c}{AMASS} & \multicolumn{4}{c}{LaFAN1} \\
\cmidrule(lr){2-5}\cmidrule(lr){6-9}
Model &
\shortstack{MPJPE/FDE\\$\downarrow$} & \shortstack{NLL\\$\downarrow$} & \shortstack{Cov95\\{\tiny(CP)}} & \shortstack{W95\\{\tiny(CP)}$\downarrow$} &
\shortstack{MPJPE/FDE\\$\downarrow$} & \shortstack{NLL\\$\downarrow$} & \shortstack{Cov95\\{\tiny(CP)}} & \shortstack{W95\\{\tiny(CP)}$\downarrow$} \\
\midrule
}
\tablehead{%
\multicolumn{9}{c}{\footnotesize\textbf{Table~\ref{tab:overview_conjbayes_all} (continued)}}\\
\toprule
& \multicolumn{4}{c}{AMASS} & \multicolumn{4}{c}{LaFAN1} \\
\cmidrule(lr){2-5}\cmidrule(lr){6-9}
Model &
\shortstack{MPJPE/FDE\\$\downarrow$} & \shortstack{NLL\\$\downarrow$} & \shortstack{Cov95\\{\tiny(CP)}} & \shortstack{W95\\{\tiny(CP)}$\downarrow$} &
\shortstack{MPJPE/FDE\\$\downarrow$} & \shortstack{NLL\\$\downarrow$} & \shortstack{Cov95\\{\tiny(CP)}} & \shortstack{W95\\{\tiny(CP)}$\downarrow$} \\
\midrule
}
\tabletail{%
\bottomrule
\multicolumn{9}{r}{\footnotesize Continued on next page}\\
}
\tablelasttail{\bottomrule}

\centering
\begin{supertabular}{@{}l *{8}{c}@{}}
siMLPe (base) & 0.0954 / 0.1230 & 13.55 & -- & -- & 0.5453 / 0.7066 & 319.27 & -- & -- \\
MeanFT (siMLPe fine-tune) & \textbf{0.0593 / 0.0840} & 3.16 & -- & -- & \underline{0.5358 / 0.6729} & 303.33 & -- & -- \\
AuxTasks (50/50) & 0.0949 / 0.1253 & 12.39 & -- & -- & 0.5827 / 0.6938 & 353.83 & -- & -- \\
Symplectic (50/50) & 0.1424 / 0.4436 & 39.70 & -- & -- & 1.4504 & 3203.83 & -- & -- \\
DeFeeNet (50/50) & 0.0909 / 0.1547 & 12.15 & -- & -- & 0.5666 & 341.82 & -- & -- \\
GSPS (50/50) & 0.1000 / 0.1318 & 14.03 & \textbf{0.950} & 0.311 & 0.5408 / 0.6939 & 282.97 & 0.953 & 1.945 \\
HumanMAC (50/50) & 0.0722 / 0.0935 & 4.34 & -- & -- & 0.5564 & 290.09 & -- & -- \\
SeSGCN (teacher) & 0.1029 / 0.1095 & 9.48 & -- & -- & 0.5872 / 0.6457 & 314.76 & -- & -- \\
Deep ensemble & 0.0954 / 0.1248 & 2.60 & \textbf{0.950} & 0.283 & 0.5388 / 0.7121 & \underline{43.45} & 0.953 & 1.856 \\
Bridge-CQR & \textbf{0.0593 / 0.0840} & \underline{0.15} & 0.954 & 0.384 & \underline{0.5358 / 0.6729} & 139.07 & \textbf{0.950} & 1.907 \\
SkeletonDiffusion (CVPR 2025) & 0.0971 / 0.1049 & 9.02 & \textbf{0.950} & 0.334 & 0.6028 / 0.6769 & 280.23 & \underline{0.951} & 4.209 \\
BeLFusion (ICCV 2023) & 0.0792 / 0.1015 & 3.84 & 0.952 & 0.235 & 0.5618 / 0.7022 & 251.70 & 0.952 & 1.894 \\
TransFusion (RA-L 2024) & 0.0790 / 0.0972 & 3.76 & \underline{0.951} & 0.198 & 0.6496 / 0.7432 & 324.90 & \textbf{0.950} & 1.847 \\
CoMusion (ECCV 2024) & \underline{0.0617 / 0.0836} & 1.49 & 0.952 & \textbf{0.174} & \textbf{0.5268 / 0.6421} & 219.66 & \underline{0.949} & \underline{1.590} \\
SPARD (AAAI 2026) & 0.0638 / 0.0880 & 1.51 & 0.953 & \underline{0.175} & 0.5682 / 0.7141 & 292.27 & 0.955 & 1.759 \\
MotionMap (CVPR 2025) & 0.2038 / 0.2012 & 16.42 & \underline{0.949} & 0.536 & 0.7463 / 0.7614 & 351.70 & \textbf{0.950} & 2.006 \\
SLD-HMP (ECCV 2024) & 0.1120 / 0.1252 & 22.45 & \underline{0.949} & 0.308 & 0.6091 / 0.7049 & 491.37 & 0.955 & 1.936 \\
SPARC$^{\dagger}$ & \textbf{0.0593 / 0.0840} & \textbf{-2.37} & 0.952 & 0.212 & \underline{0.5358 / 0.6729} & \textbf{-0.84} & 0.954 & \textbf{1.568} \\
\end{supertabular}

\par\medskip

\tablefirsthead{%
\multicolumn{9}{c}{\footnotesize\textbf{Table~\ref{tab:overview_conjbayes_all} (continued)}}\\
\toprule
& \multicolumn{4}{c}{LaFAN1-HF} & \multicolumn{4}{c}{CMU-MoCap} \\
\cmidrule(lr){2-5}\cmidrule(lr){6-9}
Model &
\shortstack{MPJPE/FDE\\$\downarrow$} & \shortstack{NLL\\$\downarrow$} & \shortstack{Cov95\\{\tiny(CP)}} & \shortstack{W95\\{\tiny(CP)}$\downarrow$} &
\shortstack{MPJPE/FDE\\$\downarrow$} & \shortstack{NLL\\$\downarrow$} & \shortstack{Cov95\\{\tiny(CP)}} & \shortstack{W95\\{\tiny(CP)}$\downarrow$} \\
\midrule
}
\tablehead{%
\multicolumn{9}{c}{\footnotesize\textbf{Table~\ref{tab:overview_conjbayes_all} (continued)}}\\
\toprule
& \multicolumn{4}{c}{LaFAN1-HF} & \multicolumn{4}{c}{CMU-MoCap} \\
\cmidrule(lr){2-5}\cmidrule(lr){6-9}
Model &
\shortstack{MPJPE/FDE\\$\downarrow$} & \shortstack{NLL\\$\downarrow$} & \shortstack{Cov95\\{\tiny(CP)}} & \shortstack{W95\\{\tiny(CP)}$\downarrow$} &
\shortstack{MPJPE/FDE\\$\downarrow$} & \shortstack{NLL\\$\downarrow$} & \shortstack{Cov95\\{\tiny(CP)}} & \shortstack{W95\\{\tiny(CP)}$\downarrow$} \\
\midrule
}
\tabletail{%
\bottomrule
\multicolumn{9}{r}{\footnotesize Continued on next page}\\
}
\tablelasttail{\bottomrule}

\centering
\begin{supertabular}{@{}l *{8}{c}@{}}
siMLPe (base) & 0.4046 / 0.6400 & 172.36 & -- & -- & 2.6560 & 8726.23 & -- & -- \\
MeanFT (siMLPe fine-tune) & \textbf{0.3093} & 99.80 & -- & -- & 2.6709 & 8820.46 & -- & -- \\
AuxTasks (50/50) & 0.4299 & 184.26 & -- & -- & \textbf{2.5429} & 8292.99 & -- & -- \\
Symplectic (50/50) & 0.4756 & 280.68 & -- & -- & 6.8024 / 10.4453 & 54096.87 & -- & -- \\
DeFeeNet (50/50) & 0.3992 / 0.6459 & 169.44 & -- & -- & \underline{2.6316} & 8633.05 & -- & -- \\
GSPS (50/50) & 0.3833 / 0.6384 & 148.57 & \textbf{0.950} & 1.225 & 2.6321 / 3.7290 & 8307.55 & 0.952 & \underline{9.785} \\
HumanMAC (50/50) & 0.4477 & 166.64 & -- & -- & 4.8433 & 21062.22 & -- & -- \\
SeSGCN (teacher) & 0.4061 / 0.5435 & 138.80 & -- & -- & 4.9092 / 5.1552 & 20996.24 & -- & -- \\
Deep ensemble & 0.3894 / 0.6260 & 80.18 & \textbf{0.950} & 1.167 & 2.7084 / 3.7452 & \underline{1310.38} & \underline{0.951} & 9.855 \\
Bridge-CQR & \underline{0.3103 / 0.4719} & \underline{45.42} & 0.962 & 1.251 & 2.7454 / 3.9290 & 4183.52 & 0.953 & 9.923 \\
SkeletonDiffusion (CVPR 2025) & 1.8071 / 1.8613 & 3357.02 & 0.959 & 32.988 & 8.0957 / 8.3124 & 49441.83 & \textbf{0.950} & 381.727 \\
BeLFusion (ICCV 2023) & 1.0006 / 1.6484 & 520.35 & 0.961 & 3.049 & 3.6186 / 4.4319 & 12261.03 & 0.954 & 14.329 \\
TransFusion (RA-L 2024) & 1.4215 / 2.0418 & 1222.72 & 0.956 & 3.015 & 8.0839 / 8.4094 & 44901.52 & 0.952 & 25.699 \\
CoMusion (ECCV 2024) & 0.3604 / 0.5383 & 92.62 & \textbf{0.950} & 1.242 & 8.6798 / 8.6127 & 54190.65 & 0.953 & 24.961 \\
SPARD (AAAI 2026) & 0.3677 / 0.5393 & 105.70 & 0.954 & \underline{1.078} & 5.5562 / 5.7267 & 25823.22 & 0.952 & 37.934 \\
MotionMap (CVPR 2025) & 1.4691 / 1.5391 & 470.72 & 0.953 & 4.645 & 5.3103 / 5.6501 & 21490.65 & \underline{0.951} & 18.384 \\
SLD-HMP (ECCV 2024) & 0.4801 / 0.6722 & 401.14 & \underline{0.952} & 1.347 & 2.7547 / 3.8620 & 16866.43 & \underline{0.951} & 11.966 \\
SPARC$^{\dagger}$ & \underline{0.3103 / 0.4719} & \textbf{-0.05} & \underline{0.952} & \textbf{0.940} & 2.7207 / 3.9083 & \textbf{9.84} & 0.953 & \textbf{8.489} \\
\end{supertabular}

\par\medskip

\tablefirsthead{%
\multicolumn{9}{c}{\footnotesize\textbf{Table~\ref{tab:overview_conjbayes_all} (continued)}}\\
\toprule
& \multicolumn{4}{c}{Human3.6M} & \multicolumn{4}{c}{3DPW} \\
\cmidrule(lr){2-5}\cmidrule(lr){6-9}
Model &
\shortstack{MPJPE/FDE\\$\downarrow$} & \shortstack{NLL\\$\downarrow$} & \shortstack{Cov95\\{\tiny(CP)}} & \shortstack{W95\\{\tiny(CP)}$\downarrow$} &
\shortstack{MPJPE/FDE\\$\downarrow$} & \shortstack{NLL\\$\downarrow$} & \shortstack{Cov95\\{\tiny(CP)}} & \shortstack{W95\\{\tiny(CP)}$\downarrow$} \\
\midrule
}
\tablehead{%
\multicolumn{9}{c}{\footnotesize\textbf{Table~\ref{tab:overview_conjbayes_all} (continued)}}\\
\toprule
& \multicolumn{4}{c}{Human3.6M} & \multicolumn{4}{c}{3DPW} \\
\cmidrule(lr){2-5}\cmidrule(lr){6-9}
Model &
\shortstack{MPJPE/FDE\\$\downarrow$} & \shortstack{NLL\\$\downarrow$} & \shortstack{Cov95\\{\tiny(CP)}} & \shortstack{W95\\{\tiny(CP)}$\downarrow$} &
\shortstack{MPJPE/FDE\\$\downarrow$} & \shortstack{NLL\\$\downarrow$} & \shortstack{Cov95\\{\tiny(CP)}} & \shortstack{W95\\{\tiny(CP)}$\downarrow$} \\
\midrule
}
\tabletail{%
\bottomrule
\multicolumn{9}{r}{\footnotesize Continued on next page}\\
}
\tablelasttail{\bottomrule}

\centering
\begin{supertabular}{@{}l *{8}{c}@{}}
siMLPe (base) & 0.1012 / 0.1535 & 12.84 & -- & -- & 0.0709 / 0.1021 & 5.08 & -- & -- \\
MeanFT (siMLPe fine-tune) & \textbf{0.0720} & 5.24 & -- & -- & \textbf{0.0572 / 0.0832} & 1.71 & -- & -- \\
AuxTasks (50/50) & 0.0999 & 12.63 & -- & -- & 0.0745 / 0.1029 & 5.27 & -- & -- \\
Symplectic (50/50) & 0.1070 / 0.2151 & 13.08 & -- & -- & 0.1042 / 0.2794 & 20.62 & -- & -- \\
DeFeeNet (50/50) & 0.0981 & 12.31 & -- & -- & 0.0701 / 0.0996 & 4.99 & -- & -- \\
GSPS (50/50) & 0.0960 / 0.1540 & 10.96 & 0.945 & 0.313 & 0.0741 / 0.1138 & 4.96 & \textbf{0.949} & 0.232 \\
HumanMAC (50/50) & 0.0982 & 7.83 & -- & -- & 0.0707 & 2.68 & -- & -- \\
SeSGCN (teacher) & 0.1119 / 0.1505 & 9.85 & -- & -- & 0.0952 / 0.1100 & 5.32 & -- & -- \\
Deep ensemble & 0.0971 / 0.1515 & 9.12 & 0.946 & 0.337 & 0.0690 / 0.1004 & 3.90 & \textbf{0.951} & 0.255 \\
Bridge-CQR & \underline{0.0721 / 0.1189} & \underline{1.13} & 0.995 & 1.238 & \underline{0.0580 / 0.0841} & \underline{-0.45} & 0.955 & 0.483 \\
SkeletonDiffusion (CVPR 2025) & 0.1242 / 0.1662 & 13.58 & 0.947 & 0.406 & 0.0822 / 0.0988 & 4.91 & \underline{0.948} & 0.256 \\
BeLFusion (ICCV 2023) & 0.1055 / 0.1512 & 6.22 & \underline{0.948} & 0.316 & 0.0771 / 0.1059 & 2.13 & \textbf{0.949} & 0.207 \\
TransFusion (RA-L 2024) & 0.1060 / 0.1449 & 7.18 & \underline{0.948} & 0.295 & 0.0787 / 0.0997 & 2.72 & 0.945 & \underline{0.187} \\
CoMusion (ECCV 2024) & 0.0944 / 0.1409 & 4.49 & 0.945 & 0.287 & 0.0690 / 0.1001 & 1.19 & 0.945 & \textbf{0.174} \\
SPARD (AAAI 2026) & 0.0928 / 0.1364 & 4.24 & 0.947 & \underline{0.250} & 0.0703 / 0.1035 & 1.25 & 0.946 & \textbf{0.174} \\
MotionMap (CVPR 2025) & 0.2504 / 0.2530 & 29.09 & \textbf{0.951} & 0.718 & 0.1841 / 0.1851 & 9.58 & \underline{0.948} & 0.476 \\
SLD-HMP (ECCV 2024) & 0.1170 / 0.1620 & 26.15 & 0.947 & 0.368 & 0.0846 / 0.1055 & 15.79 & \underline{0.948} & 0.235 \\
SPARC$^{\dagger}$ & \underline{0.0721 / 0.1189} & \textbf{-1.66} & 0.946 & \textbf{0.248} & \underline{0.0580 / 0.0841} & \textbf{-1.94} & 0.947 & 0.188 \\
\end{supertabular}

\par\medskip

\tablefirsthead{%
\multicolumn{9}{c}{\footnotesize\textbf{Table~\ref{tab:overview_conjbayes_all} (continued)}}\\
\toprule
& \multicolumn{4}{c}{CHICO} & \multicolumn{4}{c}{HA4M} \\
\cmidrule(lr){2-5}\cmidrule(lr){6-9}
Model &
\shortstack{MPJPE@10/25f\\(mm)$\downarrow$} & \shortstack{NLL\\$\downarrow$} & \shortstack{Cov95\\{\tiny(CP)}} & \shortstack{W95\\{\tiny(CP)}$\downarrow$} &
\shortstack{MPJPE@10/25f\\(mm)$\downarrow$} & \shortstack{NLL\\$\downarrow$} & \shortstack{Cov95\\{\tiny(CP)}} & \shortstack{W95\\{\tiny(CP)}$\downarrow$} \\
\midrule
}
\tablehead{%
\multicolumn{9}{c}{\footnotesize\textbf{Table~\ref{tab:overview_conjbayes_all} (continued)}}\\
\toprule
& \multicolumn{4}{c}{CHICO} & \multicolumn{4}{c}{HA4M} \\
\cmidrule(lr){2-5}\cmidrule(lr){6-9}
Model &
\shortstack{MPJPE@10/25f\\(mm)$\downarrow$} & \shortstack{NLL\\$\downarrow$} & \shortstack{Cov95\\{\tiny(CP)}} & \shortstack{W95\\{\tiny(CP)}$\downarrow$} &
\shortstack{MPJPE@10/25f\\(mm)$\downarrow$} & \shortstack{NLL\\$\downarrow$} & \shortstack{Cov95\\{\tiny(CP)}} & \shortstack{W95\\{\tiny(CP)}$\downarrow$} \\
\midrule
}
\tabletail{%
\bottomrule
\multicolumn{9}{r}{\footnotesize Continued on next page}\\
}
\tablelasttail{\bottomrule}

\centering
\begin{supertabular}{@{}l *{8}{c}@{}}
siMLPe (base) & 43.2 / 75.2 & 0.66 & -- & -- & \textbf{46.1 / 57.0} & \underline{-1.08} & -- & -- \\
MeanFT (siMLPe fine-tune) & \underline{42.7 / 64.8} & -0.35 & -- & -- & \underline{53.6 / 57.7} & -1.00 & -- & -- \\
AuxTasks (50/50) & 44.6 / 77.7 & 0.88 & -- & -- & 91.5 / 70.6 & 0.30 & -- & -- \\
Symplectic (50/50) & 46.6 / 91.2 & 0.77 & -- & -- & 79.5 / 364.9 & 24.42 & -- & -- \\
DeFeeNet (50/50) & 42.7 / 75.2 & 0.69 & -- & -- & 93.5 / 77.2 & 3.11 & -- & -- \\
GSPS (50/50) & 42.2 / 74.2 & 0.28 & 0.947 & 0.163 & 42.8 / 73.8 & -0.22 & 0.929 & \textbf{0.131} \\
HumanMAC (50/50) & 52.0 / 70.7 & 0.27 & -- & -- & 189.5 / 196.7 & 35.25 & -- & -- \\
SeSGCN (teacher) & 66.2 / 77.7 & 1.43 & -- & -- & 150.9 / 171.5 & 21.59 & -- & -- \\
Deep ensemble & 41.2 / 76.4 & -0.38 & 0.946 & 0.162 & 78.0 / 192.6 & 0.19 & 0.942 & 0.236 \\
Bridge-CQR & \underline{42.7 / 64.8} & \underline{-1.47} & 0.955 & 0.523 & 103.5 / 193.9 & 5.48 & 0.972 & 0.843 \\
SkeletonDiffusion (CVPR 2025) & 57.7 / 76.8 & 1.41 & 0.952 & 0.214 & 196.7 / 189.9 & 29.06 & 0.958 & 0.778 \\
BeLFusion (ICCV 2023) & 48.4 / 69.7 & -1.04 & \underline{0.949} & 0.150 & 144.4 / 141.6 & 8.99 & \underline{0.951} & 0.344 \\
TransFusion (RA-L 2024) & 48.6 / 66.4 & -0.81 & \underline{0.949} & 0.133 & 145.5 / 170.8 & 5.43 & 0.944 & 0.296 \\
CoMusion (ECCV 2024) & 44.8 / 65.7 & -1.28 & \textbf{0.950} & \underline{0.131} & 152.7 / 168.8 & 22.95 & \underline{0.951} & 0.547 \\
SPARD (AAAI 2026) & \textbf{44.5 / 64.5} & -1.39 & \underline{0.949} & \textbf{0.129} & 152.2 / 176.1 & 22.71 & \textbf{0.950} & 0.473 \\
MotionMap (CVPR 2025) & 161.2 / 165.5 & 3.15 & 0.953 & 0.435 & 173.2 / 166.5 & 50.54 & 0.944 & 0.431 \\
SLD-HMP (ECCV 2024) & 54.3 / 80.1 & 4.17 & 0.948 & 0.162 & 58.4 / 80.4 & 7.34 & 0.936 & \underline{0.155} \\
SPARC$^{\dagger}$ & \underline{42.7 / 64.8} & \textbf{-2.21} & \underline{0.949} & \underline{0.131} & 103.5 / 193.9 & \textbf{-2.28} & 0.981 & 0.258 \\
\end{supertabular}

\par\medskip

\tablefirsthead{%
\multicolumn{9}{c}{\footnotesize\textbf{Table~\ref{tab:overview_conjbayes_all} (continued)}}\\
\toprule
& \multicolumn{4}{c}{AnDy-onePerson} & \multicolumn{4}{c}{} \\
\cmidrule(lr){2-5}
Model &
\shortstack{MPJPE@10/25f\\(mm)$\downarrow$} & \shortstack{NLL\\$\downarrow$} & \shortstack{Cov95\\{\tiny(CP)}} & \shortstack{W95\\{\tiny(CP)}$\downarrow$} &
{} & {} & {} & {} \\
\midrule
}
\tablehead{%
\multicolumn{9}{c}{\footnotesize\textbf{Table~\ref{tab:overview_conjbayes_all} (continued)}}\\
\toprule
& \multicolumn{4}{c}{AnDy-onePerson} & \multicolumn{4}{c}{} \\
\cmidrule(lr){2-5}
Model &
\shortstack{MPJPE@10/25f\\(mm)$\downarrow$} & \shortstack{NLL\\$\downarrow$} & \shortstack{Cov95\\{\tiny(CP)}} & \shortstack{W95\\{\tiny(CP)}$\downarrow$} &
{} & {} & {} & {} \\
\midrule
}
\tabletail{%
\bottomrule
\multicolumn{9}{r}{\footnotesize Continued on next page}\\
}
\tablelasttail{\bottomrule}

\centering
\begin{supertabular}{@{}l *{8}{c}@{}}
siMLPe (base) & 28.0 / 70.3 & 2.00 & -- & -- & {} & {} & {} & {} \\
MeanFT (siMLPe fine-tune) & \textbf{27.4 / 55.7} & -0.27 & -- & -- & {} & {} & {} & {} \\
AuxTasks (50/50) & 32.5 / 70.7 & 2.22 & -- & -- & {} & {} & {} & {} \\
Symplectic (50/50) & 30.2 / 105.5 & 2.62 & -- & -- & {} & {} & {} & {} \\
DeFeeNet (50/50) & 26.1 / 72.6 & 1.75 & -- & -- & {} & {} & {} & {} \\
GSPS (50/50) & 27.9 / 73.4 & 1.89 & \underline{0.951} & 0.182 & {} & {} & {} & {} \\
HumanMAC (50/50) & 55.7 / 80.0 & 1.98 & -- & -- & {} & {} & {} & {} \\
SeSGCN (teacher) & 55.7 / 81.3 & 2.05 & -- & -- & {} & {} & {} & {} \\
Deep ensemble & 25.8 / 69.8 & 1.03 & 0.952 & 0.188 & {} & {} & {} & {} \\
Bridge-CQR & \textbf{27.4 / 55.7} & \underline{-1.44} & 0.965 & 0.866 & {} & {} & {} & {} \\
SkeletonDiffusion (CVPR 2025) & 63.3 / 89.7 & 4.60 & \underline{0.949} & 0.265 & {} & {} & {} & {} \\
BeLFusion (ICCV 2023) & 38.7 / 68.5 & -0.14 & \textbf{0.950} & 0.150 & {} & {} & {} & {} \\
TransFusion (RA-L 2024) & 54.7 / 82.1 & 0.07 & \underline{0.949} & 0.142 & {} & {} & {} & {} \\
CoMusion (ECCV 2024) & \underline{32.4 / 59.4} & -1.31 & 0.948 & \underline{0.110} & {} & {} & {} & {} \\
SPARD (AAAI 2026) & 32.9 / 61.0 & -1.40 & \underline{0.951} & \textbf{0.107} & {} & {} & {} & {} \\
MotionMap (CVPR 2025) & 179.3 / 190.2 & 15.52 & \underline{0.951} & 0.532 & {} & {} & {} & {} \\
SLD-HMP (ECCV 2024) & 52.0 / 76.7 & 7.91 & 0.948 & 0.193 & {} & {} & {} & {} \\
SPARC$^{\dagger}$ & \textbf{27.4 / 55.7} & \textbf{-3.53} & 0.953 & 0.144 & {} & {} & {} & {} \\
\end{supertabular}

\par\smallskip
{\footnotesize\textbf{Legend:} \textbf{bold} = best; \underline{underline} = second-best (ties included). For Cov95(CP), best/second-best are defined by closeness to the 0.95 target. \quad $^{\dagger}$ (MN-GraphJ, time-coupled $\kappa_t$, shrink).}
{\footnotesize Dataset-specific protocol-only baselines (e.g., CHICO-only SeSGCN protocol variants) are omitted to avoid redundant all-dataset rows.}
\endgroup

%% file: tables/ablation_conjbayes_fc_out_components.tex
\begin{table}[!htbp]
\centering
\caption{SPARC family component ablation matrix. Columns mark which components are active in each variant.}
\label{tab:ablation_conjbayes_fc_out_components}
\scriptsize
\setlength{\tabcolsep}{4pt}
\renewcommand{\arraystretch}{1.08}
\begin{adjustbox}{max width=\linewidth}
\begin{tabular}{@{}lcccccr@{}}
\toprule
Variant & DCT pinv & Timewise $\kappa_t$ & Hybrid $\Sigma_{\mathrm{ale}}$ & GraphJ & $\kappa_t$ coupling & $\mathrm{MR}_{\mathrm{MPJPE+NLL}}\downarrow$ \\
\midrule
SPARC family (base $\kappa(x)$) & -- & -- & -- & -- & -- & 5.28 \\
SPARC family (timewise $\kappa_t$) & -- & \checkmark & -- & -- & -- & 4.61 \\
SPARC family (DCT pinv) & \checkmark & -- & -- & -- & -- & 6.00 \\
SPARC family (hybrid MN) & -- & -- & \checkmark & -- & -- & 3.94 \\
SPARC family (hybrid MN-GraphJ) & -- & -- & \checkmark & \checkmark & -- & \textbf{2.47} \\
SPARC family (hybrid MN-GraphJ, timewise $\kappa_t$) & -- & \checkmark & \checkmark & \checkmark & -- & \underline{2.53} \\
SPARC (hybrid MN-GraphJ, time-coupled $\kappa_t$) & -- & \checkmark & \checkmark & \checkmark & \checkmark & 3.17 \\
\bottomrule
\end{tabular}
\end{adjustbox}
\par\smallskip
{\footnotesize\textbf{Legend:} \textbf{bold} = best; \underline{underline} = second-best.}
\end{table}

%% file: tables/ablation_conjbayes_fc_out_avg_rank.tex
\begin{table}[!htbp]
\centering
\caption{Mean rank (MR; $\downarrow$ better) across datasets within the SPARC ablation family. $\mathrm{MR}_{\mathrm{MPJPE+NLL}}$ is the mean of MPJPE and NLL ranks. $\mathrm{MR}_{W95(\mathrm{CP})}$ uses fixed-protocol Cov95/W95(\cp) re-evaluation across all datasets.}
\label{tab:ablation_conjbayes_fc_out_rank}
\scriptsize
\setlength{\tabcolsep}{5pt}
\renewcommand{\arraystretch}{1.10}
\begin{adjustbox}{max width=\linewidth}
\begin{tabular}{@{}lrrrr@{}}
\toprule
Variant & $\mathrm{MR}_{\mathrm{MPJPE}}\downarrow$ & $\mathrm{MR}_{\mathrm{NLL}}\downarrow$ & $\mathrm{MR}_{W95(\mathrm{CP})}\downarrow$ & $\mathrm{MR}_{\mathrm{MPJPE+NLL}}\downarrow$ \\
\midrule
SPARC family (base $\kappa(x)$) & 5.78 & 4.78 & 4.78 & 5.28 \\
SPARC family (timewise $\kappa_t$) & 5.78 & 3.44 & \underline{3.56} & 4.61 \\
SPARC family (DCT pinv) & 5.89 & 6.11 & 4.56 & 6.00 \\
SPARC family (hybrid MN) & \textbf{2.33} & 5.56 & 4.00 & 3.94 \\
SPARC family (hybrid MN-GraphJ) & \underline{2.50} & \textbf{2.44} & 4.78 & \textbf{2.47} \\
SPARC family (hybrid MN-GraphJ, timewise $\kappa_t$) & \underline{2.50} & \underline{2.56} & 3.78 & \underline{2.53} \\
SPARC (hybrid MN-GraphJ, time-coupled $\kappa_t$) & 3.22 & 3.11 & \textbf{2.56} & 3.17 \\
\bottomrule
\end{tabular}
\end{adjustbox}
\par\smallskip
{\footnotesize\textbf{Legend:} \textbf{bold} = best; \underline{underline} = second-best (ties included).}
\end{table}

%% file: tables/ablation_conjbayes_fc_out_all_datasets.tex
\begingroup
\setlength{\LTcapwidth}{\linewidth}
\setlength{\LTleft}{0pt}
\setlength{\LTright}{0pt}
\scriptsize
\setlength{\tabcolsep}{4pt}
\renewcommand{\arraystretch}{1.12}
\refstepcounter{table}
\noindent\textbf{Table \thetable:} Full FC-Out component-isolation ablation.\label{tab:ablation_conjbayes_fc_out_all}
\par\medskip
\begin{longtable}{@{}p{0.49\linewidth}rrrr@{}}
\toprule
Variant & MPJPE$\downarrow$ & NLL$\downarrow$ & Cov95(CP) & W95(CP)$\downarrow$ \\
\midrule
\endfirsthead
\multicolumn{5}{c}{\footnotesize\textbf{Table~\ref{tab:ablation_conjbayes_fc_out_all} (continued)}}\\
\toprule
Variant & MPJPE$\downarrow$ & NLL$\downarrow$ & Cov95(CP) & W95(CP)$\downarrow$ \\
\midrule
\endhead
\midrule
\multicolumn{5}{r}{\footnotesize Continued on next page}\\
\endfoot
\bottomrule
\endlastfoot
\addlinespace
\multicolumn{5}{@{}l@{}}{\textbf{AMASS}} \\
\addlinespace
SPARC family (base $\kappa(x)$) & 0.0549 & -1.84 & 0.949 & \underline{0.188} \\
SPARC family (timewise $\kappa_t$) & 0.0551 & \underline{-2.04} & \underline{0.949} & \textbf{0.185} \\
SPARC family (DCT pinv) & 0.0605 & -1.63 & \textbf{0.950} & 0.192 \\
SPARC family (hybrid MN) & \textbf{0.0542} & -1.64 & 0.951 & 0.189 \\
SPARC family (hybrid MN-GraphJ) & \textbf{0.0542} & -1.85 & 0.951 & 0.220 \\
SPARC family (hybrid MN-GraphJ, timewise $\kappa_t$) & \underline{0.0542} & -1.86 & 0.952 & 0.219 \\
SPARC (hybrid MN-GraphJ, time-coupled $\kappa_t$) & 0.0593 & \textbf{-2.37} & 0.952 & 0.212 \\
\addlinespace
\multicolumn{5}{@{}l@{}}{\textbf{LaFAN1}} \\
\addlinespace
SPARC family (base $\kappa(x)$) & 3.5196 & 2.37 & 0.959 & 8.788 \\
SPARC family (timewise $\kappa_t$) & 0.6858 & 1.02 & 0.957 & 2.357 \\
SPARC family (DCT pinv) & 3.4217 & 2.13 & 0.956 & 7.101 \\
SPARC family (hybrid MN) & \underline{0.5498} & 0.28 & 0.952 & \underline{1.600} \\
SPARC family (hybrid MN-GraphJ) & \underline{0.5498} & \underline{-0.38} & \textbf{0.951} & 2.062 \\
SPARC family (hybrid MN-GraphJ, timewise $\kappa_t$) & \underline{0.5498} & -0.29 & \underline{0.951} & 2.058 \\
SPARC (hybrid MN-GraphJ, time-coupled $\kappa_t$) & \textbf{0.5358} & \textbf{-0.84} & 0.954 & \textbf{1.568} \\
\addlinespace
\multicolumn{5}{@{}l@{}}{\textbf{LaFAN1-HF}} \\
\addlinespace
SPARC family (base $\kappa(x)$) & 0.4995 & 1.03 & 0.951 & 1.444 \\
SPARC family (timewise $\kappa_t$) & 0.5271 & 3.72 & \textbf{0.950} & 1.444 \\
SPARC family (DCT pinv) & 0.4417 & 1.10 & 0.953 & 1.326 \\
SPARC family (hybrid MN) & \underline{0.3093} & 3.58 & 0.952 & 0.978 \\
SPARC family (hybrid MN-GraphJ) & \underline{0.3093} & \textbf{-0.07} & \underline{0.951} & 0.996 \\
SPARC family (hybrid MN-GraphJ, timewise $\kappa_t$) & \textbf{0.3093} & 0.02 & 0.951 & \underline{0.975} \\
SPARC (hybrid MN-GraphJ, time-coupled $\kappa_t$) & 0.3103 & \underline{-0.05} & 0.952 & \textbf{0.940} \\
\addlinespace
\multicolumn{5}{@{}l@{}}{\textbf{CMU-MoCap}} \\
\addlinespace
SPARC family (base $\kappa(x)$) & 4.4830 & 2.40 & \textbf{0.952} & 17.165 \\
SPARC family (timewise $\kappa_t$) & 2.7509 & \textbf{2.12} & 0.953 & 9.685 \\
SPARC family (DCT pinv) & 21.8957 & 3.19 & 0.956 & 39.964 \\
SPARC family (hybrid MN) & \textbf{2.6549} & 3.67 & 0.953 & 10.252 \\
SPARC family (hybrid MN-GraphJ) & \underline{2.6836} & \underline{2.18} & \underline{0.952} & 9.932 \\
SPARC family (hybrid MN-GraphJ, timewise $\kappa_t$) & \textbf{2.6549} & 2.77 & 0.953 & \underline{9.388} \\
SPARC (hybrid MN-GraphJ, time-coupled $\kappa_t$) & 2.7207 & 9.84 & 0.953 & \textbf{8.489} \\
\addlinespace
\multicolumn{5}{@{}l@{}}{\textbf{Human3.6M}} \\
\addlinespace
SPARC family (base $\kappa(x)$) & 0.0729 & -1.42 & 0.944 & \underline{0.241} \\
SPARC family (timewise $\kappa_t$) & 0.0730 & \textbf{-1.72} & 0.944 & \textbf{0.240} \\
SPARC family (DCT pinv) & 0.0730 & -1.18 & 0.944 & 0.244 \\
SPARC family (hybrid MN) & \textbf{0.0720} & -1.29 & \underline{0.945} & 0.253 \\
SPARC family (hybrid MN-GraphJ) & \textbf{0.0720} & -1.69 & 0.945 & 0.251 \\
SPARC family (hybrid MN-GraphJ, timewise $\kappa_t$) & \textbf{0.0720} & \underline{-1.69} & 0.945 & 0.251 \\
SPARC (hybrid MN-GraphJ, time-coupled $\kappa_t$) & \underline{0.0721} & -1.66 & \textbf{0.946} & 0.248 \\
\addlinespace
\multicolumn{5}{@{}l@{}}{\textbf{3DPW}} \\
\addlinespace
SPARC family (base $\kappa(x)$) & 0.0599 & -1.86 & 0.946 & \underline{0.168} \\
SPARC family (timewise $\kappa_t$) & 0.0599 & -2.06 & 0.947 & \textbf{0.165} \\
SPARC family (DCT pinv) & 0.0595 & -1.90 & \underline{0.947} & 0.169 \\
SPARC family (hybrid MN) & \underline{0.0581} & -1.87 & 0.946 & 0.173 \\
SPARC family (hybrid MN-GraphJ) & \underline{0.0581} & \underline{-2.48} & \textbf{0.948} & 0.195 \\
SPARC family (hybrid MN-GraphJ, timewise $\kappa_t$) & \underline{0.0581} & \textbf{-2.48} & \textbf{0.948} & 0.195 \\
SPARC (hybrid MN-GraphJ, time-coupled $\kappa_t$) & \textbf{0.0580} & -1.94 & 0.947 & 0.188 \\
\addlinespace
\multicolumn{5}{@{}l@{}}{\textbf{CHICO}} \\
\addlinespace
SPARC family (base $\kappa(x)$) & 48.6 / 74.1 & \textbf{-2.31} & 0.948 & 0.140 \\
SPARC family (timewise $\kappa_t$) & 49.1 / 74.0 & \underline{-2.30} & \underline{0.948} & 0.139 \\
SPARC family (DCT pinv) & 46.4 / 69.6 & -2.07 & 0.948 & \underline{0.137} \\
SPARC family (hybrid MN) & \underline{43.7 / 65.5} & -1.68 & 0.947 & 0.154 \\
SPARC family (hybrid MN-GraphJ) & \underline{43.7 / 65.5} & -2.21 & 0.947 & 0.151 \\
SPARC family (hybrid MN-GraphJ, timewise $\kappa_t$) & 43.7 / 65.5 & -2.23 & 0.948 & 0.148 \\
SPARC (hybrid MN-GraphJ, time-coupled $\kappa_t$) & \textbf{42.7 / 64.8} & -2.21 & \textbf{0.949} & \textbf{0.131} \\
\addlinespace
\multicolumn{5}{@{}l@{}}{\textbf{HA4M}} \\
\addlinespace
SPARC family (base $\kappa(x)$) & 57.7 / 130.4 & -1.28 & \underline{0.949} & 0.482 \\
SPARC family (timewise $\kappa_t$) & 41.3 / 83.9 & -1.09 & \textbf{0.951} & 0.433 \\
SPARC family (DCT pinv) & 56.0 / 148.6 & -0.86 & 0.951 & 0.402 \\
SPARC family (hybrid MN) & \textbf{42.0 / 76.7} & -1.36 & 0.945 & \underline{0.307} \\
SPARC family (hybrid MN-GraphJ) & \textbf{42.0 / 76.7} & \textbf{-2.55} & 0.945 & 0.355 \\
SPARC family (hybrid MN-GraphJ, timewise $\kappa_t$) & \underline{42.0 / 76.7} & \underline{-2.41} & 0.946 & 0.356 \\
SPARC (hybrid MN-GraphJ, time-coupled $\kappa_t$) & 103.5 / 193.9 & -2.28 & 0.981 & \textbf{0.258} \\
\addlinespace
\multicolumn{5}{@{}l@{}}{\textbf{AnDy-onePerson}} \\
\addlinespace
SPARC family (base $\kappa(x)$) & 55.2 / 83.5 & -1.44 & 0.948 & 0.196 \\
SPARC family (timewise $\kappa_t$) & 59.3 / 129.6 & -2.17 & 0.948 & 0.179 \\
SPARC family (DCT pinv) & 48.3 / 75.7 & -0.81 & 0.948 & 0.187 \\
SPARC family (hybrid MN) & 30.7 / 58.6 & -1.65 & \textbf{0.951} & 0.136 \\
SPARC family (hybrid MN-GraphJ) & 30.7 / 58.6 & -2.69 & 0.948 & \underline{0.124} \\
SPARC family (hybrid MN-GraphJ, timewise $\kappa_t$) & \underline{30.7 / 58.6} & \underline{-2.69} & \underline{0.948} & \textbf{0.115} \\
SPARC (hybrid MN-GraphJ, time-coupled $\kappa_t$) & \textbf{27.4 / 55.7} & \textbf{-3.53} & 0.953 & 0.144 \\
\end{longtable}
\par\smallskip
{\footnotesize Standard protocol: seed=304, $n_{\mathrm{cal}}=512$, $n_{\mathrm{eval}}=1024$, $\alpha=0.05$. CP columns use fixed-protocol Cov95/W95(\cp) re-evaluation across all datasets for every reported FC-Out ablation variant.}
\endgroup

%% file: tables/conjgraph_variants_all_datasets.tex
\begingroup
\setlength{\LTcapwidth}{\linewidth}
\setlength{\LTleft}{\fill}
\setlength{\LTright}{\fill}
\scriptsize
\setlength{\tabcolsep}{1.8pt}
\renewcommand{\arraystretch}{1.08}
\begin{longtable}{@{}lrrrrrrr@{}}
\caption{SPARC conformal variants across datasets (target 95\%). CP denotes split conformal marginal tubes; Mondrian denotes $\kappa$-conditioned split conformal tubes; Set denotes the trajectory-level Mahalanobis conformal ellipsoid. $\log$Vol$/d$ is the mean log-volume per dimension (lower is tighter).}\label{tab:conjgraph_variants_all}\\
\toprule
Dataset & \shortstack{Cov95\\(CP)} & \shortstack{W95(CP)\\$\downarrow$} & \shortstack{Cov95\\(Mondrian)} & \shortstack{W95(Mondrian)\\$\downarrow$} & \shortstack{Cov95\\(Set)} & \shortstack{$\log$Vol$_{\mathrm{tube}}/d$\\$\downarrow$} & \shortstack{$\log$Vol$_{\mathrm{set}}/d$\\$\downarrow$} \\
\midrule
\endfirsthead
\caption[]{SPARC conformal variants across datasets. (continued)}\\
\toprule
Dataset & \shortstack{Cov95\\(CP)} & \shortstack{W95(CP)\\$\downarrow$} & \shortstack{Cov95\\(Mondrian)} & \shortstack{W95(Mondrian)\\$\downarrow$} & \shortstack{Cov95\\(Set)} & \shortstack{$\log$Vol$_{\mathrm{tube}}/d$\\$\downarrow$} & \shortstack{$\log$Vol$_{\mathrm{set}}/d$\\$\downarrow$} \\
\midrule
\endhead
\midrule
\multicolumn{8}{r}{\footnotesize Continued on next page}\\
\endfoot
\bottomrule
\endlastfoot
AMASS & 0.952 & 0.212 & 0.956 & 0.209 & 0.958 & -1.902 & -1.810 \\
LaFAN1 & 0.954 & 1.568 & 0.955 & 1.587 & 0.965 & 0.117 & -0.456 \\
LaFAN1-HF & 0.952 & 0.940 & 0.952 & 0.910 & 0.938 & -0.729 & 0.332 \\
CMU-MoCap & 0.953 & 8.489 & 0.955 & 8.322 & 0.953 & 1.620 & 2.452 \\
Human3.6M & 0.946 & 0.248 & 0.948 & 0.248 & 0.954 & -1.714 & -1.316 \\
3DPW & 0.947 & 0.188 & 0.949 & 0.183 & 0.951 & -2.072 & -1.536 \\
CHICO & 0.949 & 0.131 & 0.952 & 0.132 & 0.936 & -2.751 & -1.830 \\
HA4M & 0.981 & 0.258 & 0.925 & 0.231 & 0.923 & -1.905 & -2.376 \\
AnDy-onePerson & 0.953 & 0.144 & 0.951 & 0.110 & 0.971 & -2.629 & -2.781 \\
\end{longtable}
\endgroup

%% file: tables/conjbayes_kappa_selective_h36m.tex
% Auto-generated by scripts/paper/make_kappa_selective_table.py
\begingroup
\setlength{\LTcapwidth}{\linewidth}
\setlength{\LTleft}{\fill}
\setlength{\LTright}{\fill}
\scriptsize
\setlength{\tabcolsep}{4.0pt}
\renewcommand{\arraystretch}{1.10}
\begin{longtable}{@{}lrrrr@{}}
\caption{Selective prediction using the analytic epistemic scale $\bar{\kappa}(x)$ on Human3.6M ($H{=}25$; seed=304, $n_{\mathrm{eval}}{=}1024$). We sort evaluation sequences by mean $\sqrt{\kappa_t(x)}$ and report MPJPE on subsets induced by abstention/fallback thresholds (keep lowest-$\kappa$ vs.\ worst highest-$\kappa$). This is a post-hoc analysis (point predictions are unchanged) and is not used in the main ranking tables; it provides no formal guarantee, but illustrates how $\kappa$ can prioritize difficult windows. Pearson correlation between mean $\sqrt{\kappa_t(x)}$ and MPJPE is $r{=}0.52$.}\label{tab:kappa_selective}\\
\toprule
Subset & Fraction & \#Seq. & \shortstack{MPJPE\\$\downarrow$} & Rel. \\
\midrule
\endfirsthead
\caption[]{Selective prediction using $\bar{\kappa}(x)$. (continued)}\\
\toprule
Subset & Fraction & \#Seq. & \shortstack{MPJPE\\$\downarrow$} & Rel. \\
\midrule
\endhead
\midrule
\multicolumn{5}{r}{\footnotesize Continued on next page}\\
\endfoot
\bottomrule
\endlastfoot
All (baseline) & 100\% & 1024 & 0.0721 & 1.00 \\
\midrule
Keep lowest-$\kappa$ & 90\% & 922 & 0.0658 & 0.91 \\
Keep lowest-$\kappa$ & 80\% & 820 & 0.0622 & 0.86 \\
Keep lowest-$\kappa$ & 70\% & 717 & 0.0605 & 0.84 \\
Keep lowest-$\kappa$ & 60\% & 615 & 0.0596 & 0.83 \\
Keep lowest-$\kappa$ & 50\% & 512 & 0.0580 & 0.80 \\
\midrule
Worst highest-$\kappa$ & 20\% & 205 & 0.1118 & 1.55 \\
Worst highest-$\kappa$ & 10\% & 103 & 0.1292 & 1.79 \\
\end{longtable}
\endgroup

%% file: tables/h36m_conjgraph_kappa_mondrian.tex
\begingroup
\setlength{\LTcapwidth}{\linewidth}
\setlength{\LTleft}{\fill}
\setlength{\LTright}{\fill}
\scriptsize
\setlength{\tabcolsep}{2.4pt}
\renewcommand{\arraystretch}{1.08}
\begin{longtable}{@{}lrrrrrr@{}}
\caption{Human3.6M: $\kappa$-conditioned (Mondrian) split conformal for marginal tubes. Protocol: calibrate on held-out windows (n=512), evaluate on disjoint test windows (n=1024), target $\alpha=0.05$, and form 3 calibration-set quantile bins by mean $\sqrt{\kappa_t(x)}$. The table compares unconditioned CP with Mondrian bin-wise calibration using marginal coverage (Cov95) and mean interval width (W95).}\label{tab:h36m_kappa_mondrian}\\
\toprule
$\sqrt{\kappa}$ bin & \#Seq. & Cov95(CP) & W95(CP)$\downarrow$ & Cov95(Mondrian) & W95(Mondrian)$\downarrow$ \\
\midrule
\endfirsthead
\caption[]{Human3.6M: $\kappa$-conditioned (Mondrian) split conformal for marginal tubes. (continued)}\\
\toprule
$\sqrt{\kappa}$ bin & \#Seq. & Cov95(CP) & W95(CP)$\downarrow$ & Cov95(Mondrian) & W95(Mondrian)$\downarrow$ \\
\midrule
\endhead
\midrule
\multicolumn{6}{r}{\footnotesize Continued on next page}\\
\endfoot
\bottomrule
\endlastfoot
0--33\% [1.000175,1.001171) & 290 & 0.971 & 0.248 & 0.944 & 0.187 \\
33--67\% [1.001171,1.003041) & 370 & 0.958 & 0.248 & 0.950 & 0.231 \\
67--100\% [1.003041,1.012778) & 364 & 0.912 & 0.249 & 0.950 & 0.315 \\
\midrule
All & 1024 & 0.946 & 0.248 & 0.948 & 0.248 \\
\end{longtable}
\endgroup

%% file: tables/h36m_conjgraph_traj_set.tex
\begingroup
\setlength{\LTcapwidth}{\linewidth}
\setlength{\LTleft}{\fill}
\setlength{\LTright}{\fill}
\scriptsize
\setlength{\tabcolsep}{3.6pt}
\renewcommand{\arraystretch}{1.08}
\begin{longtable}{@{}lrr@{}}
\caption{Human3.6M: trajectory-level conformal ellipsoid from the structured hybrid covariance. Protocol: $n_{\mathrm{cal}}=512$, $n_{\mathrm{eval}}=1024$, $\alpha=0.05$. The table compares the axis-aligned CP rectangle with the conformal ellipsoid induced by the structured Mahalanobis/whitened-residual score under $\Sigma_T^{\mathrm{hyb}}\otimes\Sigma_C$; $\log$Vol$/d$ is mean log-volume per dimension (lower is tighter).}\label{tab:h36m_traj_set}\\
\toprule
Set & Coverage & \shortstack{$\log$Vol$/d$\\$\downarrow$} \\
\midrule
\endfirsthead
\caption[]{Human3.6M: trajectory-level conformal ellipsoid from the structured hybrid covariance. (continued)}\\
\toprule
Set & Coverage & \shortstack{$\log$Vol$/d$\\$\downarrow$} \\
\midrule
\endhead
\midrule
\multicolumn{3}{r}{\footnotesize Continued on next page}\\
\endfoot
\bottomrule
\endlastfoot
Marginal tube (CP) & 0.946 & -1.714 \\
Trajectory ellipsoid (Mahalanobis CP) & 0.954 & -1.316 \\
\end{longtable}
\endgroup

%% file: tables/synthetic_conjgraph_tuning_summary.tex
% Auto-generated by scripts/paper/make_synthetic_tuning_summary_table.py
\begingroup
\setlength{\LTcapwidth}{\linewidth}
\setlength{\LTleft}{\fill}
\setlength{\LTright}{\fill}
\scriptsize
\setlength{\tabcolsep}{2.0pt}
\begin{longtable}{@{}p{0.22\linewidth}ccccccc@{}}
\caption{Before/after synthetic tuning summary for SPARC under the fixed protocol split ($n_{\mathrm{cal}}{=}512$, $n_{\mathrm{eval}}{=}1024$, $\alpha{=}0.05$). ``Scale'' denotes the fixed multiplier applied to the analytic Bayesian quadratic form in $\kappa_t$. Cov95(CP) uses per-joint split conformal quantiles $q_{t,j}$ (pooled over the 3 coordinates of each joint).}\label{tab:synth_tuning_summary}\\
\toprule
Setting & \shortstack{$r_{\kappa}$\\ID $\uparrow$} & \shortstack{$r_{\kappa}$\\shift $\uparrow$} & \shortstack{Temp\\MAE $\downarrow$} & \shortstack{Joint\\MAE $\downarrow$} & \shortstack{Cov95(CP)\\ID} & \shortstack{Cov95(CP)\\shift} & \shortstack{NLL\\$\downarrow$} \\
\midrule
\endfirsthead
\caption[]{Before/after synthetic tuning summary for SPARC. (continued)}\\
\toprule
Setting & \shortstack{$r_{\kappa}$\\ID $\uparrow$} & \shortstack{$r_{\kappa}$\\shift $\uparrow$} & \shortstack{Temp\\MAE $\downarrow$} & \shortstack{Joint\\MAE $\downarrow$} & \shortstack{Cov95(CP)\\ID} & \shortstack{Cov95(CP)\\shift} & \shortstack{NLL\\$\downarrow$} \\
\midrule
\endhead
\midrule
\multicolumn{8}{r}{\footnotesize Continued on next page}\\
\endfoot
\bottomrule
\endlastfoot
Legacy (no feat-std, scale 1) & 0.475 & 0.820 & \textbf{0.022} & 0.103 & 0.950 & 0.667 & -2.638 \\
Tuned (feat-std, scale 200) & \textbf{0.499} & \textbf{0.848} & 0.022 & \textbf{0.100} & 0.949 & 0.847 & \textbf{-2.645} \\
\end{longtable}
\endgroup

%% file: tables/synthetic_conjgraph_kappa_mondrian.tex
\begingroup
\setlength{\LTcapwidth}{\linewidth}
\setlength{\LTleft}{\fill}
\setlength{\LTright}{\fill}
\scriptsize
\setlength{\tabcolsep}{2.8pt}
\renewcommand{\arraystretch}{1.08}
\begin{longtable}{@{}lrrrrrr@{}}
\caption{Synthetic ConjGraph benchmark: $\kappa$-conditioned (Mondrian) split conformal for marginal tubes. Protocol: split the saved ID evaluation set (n=1024) into calibration/test halves (512/512) and form 5 calibration-set quantile bins by mean $\sqrt{\kappa_t(x)}$ with target $\alpha=0.05$. The table compares unconditioned CP with Mondrian bin-wise calibration on ID and shifted evaluations using marginal coverage (Cov95) and mean interval width (W95).}\label{tab:synth_kappa_mondrian}\\
\toprule
$\sqrt{\kappa}$ bin & \#Seq. & \shortstack{Cov95\\(CP)} & \shortstack{W95(CP)\\$\downarrow$} & \shortstack{Cov95\\(Mondrian)} & \shortstack{W95(Mondrian)\\$\downarrow$} \\
\midrule
\endfirsthead
\caption[]{Synthetic ConjGraph benchmark: $\kappa$-conditioned (Mondrian) split conformal for marginal tubes. (continued)}\\
\toprule
$\sqrt{\kappa}$ bin & \#Seq. & \shortstack{Cov95\\(CP)} & \shortstack{W95(CP)\\$\downarrow$} & \shortstack{Cov95\\(Mondrian)} & \shortstack{W95(Mondrian)\\$\downarrow$} \\
\midrule
\endhead
\midrule
\multicolumn{6}{r}{\footnotesize Continued on next page}\\
\endfoot
\bottomrule
\endlastfoot
\multicolumn{6}{@{}l@{}}{\textbf{In-distribution (ID) test}} \\
0--20\% [1.03,1.05) & 107 & 0.959 & 0.146 & 0.958 & 0.146 \\
20--40\% [1.05,1.07) & 133 & 0.954 & 0.149 & 0.957 & 0.151 \\
40--60\% [1.07,1.09) & 93 & 0.954 & 0.152 & 0.961 & 0.158 \\
60--80\% [1.09,1.13) & 87 & 0.946 & 0.156 & 0.952 & 0.162 \\
80--100\% [1.13,1.65) & 92 & 0.953 & 0.175 & 0.954 & 0.176 \\
\midrule
All & 512 & 0.954 & 0.155 & 0.957 & 0.158 \\
\midrule
\multicolumn{6}{@{}l@{}}{\textbf{Shifted evaluation}} \\
0--20\% [1.03,1.05) & 2 & 0.938 & 0.147 & 0.932 & 0.147 \\
20--40\% [1.05,1.07) & 9 & 0.962 & 0.149 & 0.966 & 0.151 \\
40--60\% [1.07,1.09) & 18 & 0.963 & 0.152 & 0.970 & 0.158 \\
60--80\% [1.09,1.13) & 39 & 0.948 & 0.156 & 0.955 & 0.162 \\
80--100\% [1.13,1.65) & 956 & 0.844 & 0.281 & 0.846 & 0.283 \\
\midrule
All & 1024 & 0.851 & 0.272 & 0.853 & 0.275 \\
\end{longtable}
\endgroup

%% file: tables/synthetic_conjgraph_traj_set.tex
\begingroup
\setlength{\LTcapwidth}{\linewidth}
\setlength{\LTleft}{\fill}
\setlength{\LTright}{\fill}
\scriptsize
\setlength{\tabcolsep}{2.8pt}
\renewcommand{\arraystretch}{1.08}
\begin{longtable}{@{}lrrrr@{}}
\caption{Synthetic ConjGraph benchmark: trajectory-level conformal set from the structured hybrid covariance. Protocol: saved ID evaluation split halved into calibration/test subsets (512/512), $\alpha=0.05$; shifted results ($n=1024$) are reference-only because exchangeability does not apply. The table compares the axis-aligned CP rectangle with the ellipsoid from the whitened Frobenius/Mahalanobis score $\| (L_T^{\mathrm{hyb}})^{-1}(y-\mu)L_C^{-\top}\|_F$ under $\Sigma_T^{\mathrm{hyb}}\otimes\Sigma_C$; $\log$Vol$/d$ is mean log-volume per dimension (lower is tighter).}\label{tab:synth_traj_set}\\
\toprule
Set & \shortstack{Cov95\\ID} & \shortstack{Cov95\\shift} & \shortstack{$\log$Vol$/d$\\ID $\downarrow$} & \shortstack{$\log$Vol$/d$\\shift $\downarrow$} \\
\midrule
\endfirsthead
\caption[]{Synthetic ConjGraph benchmark: trajectory-level conformal set from the structured hybrid covariance. (continued)}\\
\toprule
Set & \shortstack{Cov95\\ID} & \shortstack{Cov95\\shift} & \shortstack{$\log$Vol$/d$\\ID $\downarrow$} & \shortstack{$\log$Vol$/d$\\shift $\downarrow$} \\
\midrule
\endhead
\midrule
\multicolumn{5}{r}{\footnotesize Continued on next page}\\
\endfoot
\bottomrule
\endlastfoot
Marginal tube (CP) & 0.954 & 0.851 & -1.870 & -1.357 \\
Trajectory ellipsoid (Mahalanobis CP) & 0.967 & 0.275 & -2.546 & -2.032 \\
\end{longtable}
\endgroup

%% file: tables/runtime_structured.tex
% Auto-generated by scripts/paper/bench_runtime_structured.py
\setlength{\LTcapwidth}{\linewidth}
\setlength{\LTleft}{\fill}
\setlength{\LTright}{\fill}
\scriptsize
\setlength{\tabcolsep}{4.0pt}
\renewcommand{\arraystretch}{1.10}
\begin{longtable}{@{}lrrr@{}}
\caption{Single-sample inference latency for structured uncertainty heads (Human3.6M, $H{=}25$, batch size 1).}\label{tab:runtime_structured}\\
\toprule
Model & \shortstack{CPU\\(ms)} & \shortstack{GPU\\(ms)} & \shortstack{GPU peak\\(MB)} \\
\midrule
\endfirsthead
\caption[]{Single-sample inference latency for structured uncertainty heads. (continued)}\\
\toprule
Model & \shortstack{CPU\\(ms)} & \shortstack{GPU\\(ms)} & \shortstack{GPU peak\\(MB)} \\
\midrule
\endhead
\midrule
\multicolumn{4}{r}{\footnotesize Continued on next page}\\
\endfoot
\bottomrule
\endlastfoot
siMLPe (MeanFT) & 1.21 & 2.36 & 9.8 \\
MatrixNormal & 1.42 & 2.83 & 9.9 \\
MatrixNormalGraph & 1.44 & 2.91 & 9.9 \\
GMRF & 1.31 & 2.67 & 9.8 \\
\end{longtable}

%% file: tables/appendix_compute_cost.tex
% Auto-generated by scripts/paper/bench_appendix_compute_cost.py
\setlength{\LTcapwidth}{\linewidth}
\setlength{\LTleft}{\fill}
\setlength{\LTright}{\fill}
\scriptsize
\setlength{\tabcolsep}{2.5pt}
\renewcommand{\arraystretch}{1.10}
\begin{longtable}{@{}p{0.34\linewidth}rrrrrr@{}}
\caption{Compute and storage overhead (Human3.6M, $H{=}25$, batch size 1). We report the number of samples $K$ used to form predictive mean/std (stochastic baselines), total parameter count, additional stored non-parameter state (e.g., SPARC: $\Lambda^{-1}$ statistics), and single-sample inference latency. External baselines: SkeletonDiffusion~\citep{curreli2025nonisotropic}, BeLFusion~\citep{barquero2023belfusion}, TransFusion~\citep{tian2024transfusion}, CoMusion~\citep{sun2024comusion}, SPARD~\citep{zhang2026spard}, MotionMap~\citep{hosseininejad2025motionmap}, SLD-HMP~\citep{xu2024learning}.}\label{tab:appendix_compute_cost}\\
\toprule
Model & $K$ & \shortstack{Params\\(M)} & \shortstack{Extra state\\(MB)} & \shortstack{CPU\\(ms)} & \shortstack{GPU\\(ms)} & \shortstack{GPU peak\\(MB)} \\
\midrule
\endfirsthead
\caption[]{Compute and storage overhead (continued)}\\
\toprule
Model & $K$ & \shortstack{Params\\(M)} & \shortstack{Extra state\\(MB)} & \shortstack{CPU\\(ms)} & \shortstack{GPU\\(ms)} & \shortstack{GPU peak\\(MB)} \\
\midrule
\endhead
\midrule
\multicolumn{7}{r}{\footnotesize Continued on next page}\\
\endfoot
\bottomrule
\endlastfoot
siMLPe (base) & 1 & 0.14 & 0.00 & 1.23 & 2.38 & 9.8 \\
MeanFT (siMLPe fine-tune) & 1 & 0.14 & 0.00 & 1.22 & 2.39 & 9.8 \\
AuxTasks (50/50) & 1 & 0.14 & 0.00 & 1.23 & 2.39 & 9.8 \\
Symplectic (50/50) & 1 & 0.14 & 0.00 & 1.34 & 2.63 & 9.8 \\
DeFeeNet (50/50) & 1 & 0.16 & 0.00 & 1.46 & 2.79 & 9.9 \\
GSPS (50/50) & 7 & 0.16 & 0.00 & 1.32 & 2.49 & 10.0 \\
HumanMAC (50/50) & 1 & 0.29 & 0.00 & 1.52 & 2.36 & 10.4 \\
SeSGCN (teacher) & 1 & 0.44 & 0.00 & 2.45 & 0.64 & 11.0 \\
Deep ensemble & 5 & 0.69 & 0.00 & 6.41 & 12.08 & 12.1 \\
Bridge-CQR & 1 & 0.15 & 0.00 & 1.29 & 2.48 & 12.7 \\
MatrixNormal & 1 & 0.17 & 0.00 & 1.43 & 2.86 & 12.8 \\
MatrixNormalGraph & 1 & 0.16 & 0.00 & 1.45 & 2.92 & 12.7 \\
GMRF & 1 & 0.14 & 0.00 & 1.32 & 2.68 & 12.7 \\
SPARC & 1 & 0.16 & 0.45 & 2.81 & 5.49 & 13.2 \\
SkeletonDiffusion (CVPR 2025) & 20 & 0.31 & 0.00 & -- & 19.13 & 14.5 \\
BeLFusion (ICCV 2023) & 20 & 0.71 & 0.00 & -- & 0.95 & 15.5 \\
TransFusion (RA-L 2024) & 20 & 15.73 & 0.00 & -- & 36.59 & 80.9 \\
CoMusion (ECCV 2024) & 20 & 16.16 & 0.00 & -- & 16.05 & 77.9 \\
SPARD (AAAI 2026) & 20 & 32.03 & 0.00 & -- & 35.52 & 145.3 \\
MotionMap (CVPR 2025) & 20 & 2.23 & 0.50 & -- & 3.72 & 22.3 \\
SLD-HMP (ECCV 2024) & 20 & 0.03 & 0.00 & -- & 6.85 & 11.9 \\
\end{longtable}